\documentclass[runningheads,a4paper]{llncs}

\usepackage{amssymb}
\usepackage[breaklinks=true,bookmarks=false]{hyperref}
\usepackage{url}
\usepackage{amsmath,amssymb} % define this before the line numbering.
\usepackage{epsfig}
\usepackage{import}
\usepackage{bm}
\usepackage{multirow}
\usepackage{algorithmic}
\usepackage{xspace}
\usepackage[font=small,labelfont=bf]{caption}
\usepackage{color}
\usepackage{tabularx}
\usepackage{array}
\usepackage[export]{adjustbox}

\makeatletter
\DeclareRobustCommand\onedot{\futurelet\@let@token\@onedot}
\def\@onedot{\ifx\@let@token.\else.\null\fi\xspace}
 
\def\ie{\emph{i.e}\onedot}

\makeatother

\newcommand{\RR}{\mathbb{R}}

\newcommand{\ba}{\mathbf{a}}

\newcommand{\bl}{\mathbf{l}}

\newcommand{\bn}{\mathbf{n}}

\newcommand{\bx}{\mathbf{x}}

\begin{document}
\mainmatter

\title{A Variational Approach to Shape-from-shading Under Natural Illumination}
\titlerunning{A Variational Approach to Shape-from-shading Under Natural Illumination}

\author{Yvain \textsc{Qu\'eau}\inst{1} \and Jean \textsc{M\'elou}\inst{2,3} \and \\ Fabien \textsc{Castan}\inst{3} \and Daniel \textsc{Cremers}\inst{1} \and Jean-Denis \textsc{Durou}\inst{2}}
\authorrunning{Yvain Qu\'eau et al.}

\institute{
Department of Informatics, Technical University Munich, Germany \\
\and
IRIT, UMR CNRS 5505, Universit\'e de Toulouse, France \\
\and
Mikros Image, Levallois-Perret, France}

\maketitle

\begin{abstract}
  A numerical solution to shape-from-shading under natural illumination is presented. It builds upon an augmented Lagrangian approach for solving a generic PDE-based shape-from-shading model which handles directional or spherical harmonic lighting, orthographic or perspective projection, and greylevel or multi-channel images. 
  %An augmented Lagrangian solver is proposed, which separates the global, nonlinear optimization problem into a sequence of simpler (either local or linear) ones. 
  Real-world applications to shading-aware depth map denoising, refinement and completion are presented. 
\end{abstract}

%%%%%%%%% BODY TEXT

\section{Introduction}

Standard 3D-reconstruction pipelines are based on sparse 3D-reconstruction by structure-from-motion (SFM), densified by multi-view stereo (MVS). Both these techniques require unambiguous correspondences based on local color variations. Assumptions behind this requirement are that the surface of interest is Lambertian and well textured.
% without too many repetitive patterns. 
This has proved to be suitable for sparse reconstruction, but problematic for dense reconstruction: dense matching is impossible in textureless areas. In contrast, shape-from-shading (SFS) techniques explicitly model the reflectance of the object surface. The brightness variations observed in a single image provide dense geometric clues, even in textureless areas. SFS may thus eventually push back the limits of MVS.

However, most shape-from-shading methods require a highly controlled illumination and thus may fail when deployed outside the lab. \textbf{Numerical methods for SFS under natural illumination are still lacking}. Besides, SFS remains a classic ill-posed problem with well-known ambiguities such as the concave$/$convex ambiguity. Solving such ambiguities for real-world applications requires \textbf{handling priors on the surface}. There exist two main numerical strategies for solving shape-from-shading~\cite{Durou2008}. Variational methods~\cite{Horn1986} ensure smoothness through regularization. Handling priors is easy, but tuning the regularization may be tedious. Alternatively, methods based on the exact resolution of a nonlinear PDE~\cite{Lions1993}, which implicitly enforce differentiability (almost everywhere), do not require any tuning, but they lack robustness and they require a boundary condition. To combine the advantages of each approach, \textbf{a variational solution based on PDEs would be worthwile for SFS under natural illumination}.

\paragraph{Contributions --} This work proposes a generic numerical framework for SFS under natural illumination, which can be employed to achieve either pure shape-from-shading or shading-aware depth refinement (see Figure~\ref{fig:teaser}). After reviewing existing solutions in Section~\ref{sec:2}, we introduce in Section~\ref{sec:3} a new PDE-based model for SFS, which handles various illumination and camera models. A variational approach for solving the arising PDE is proposed in Section~\ref{sec:4}, which includes optional regularization terms for incorporating a shape prior or enforcing smoothness. Numerical solving is carried out using an ADMM algorithm.
% It reformulates SFS as a sequence of easier subproblems: local estimation of the surface normals (possibly, with a smoothness prior), and then integration of surface normals into a depth map (possibly, with a shape prior). 
Experiments on synthetic datasets are presented in Section~\ref{sec:5}, as well as real-world applications to depth refinement and completion for RGB-D cameras or stereovision systems. Our achievements are eventually summarized in Section~\ref{sec:6}.

\begin{figure}[!ht]
  \begin{tabular}{ccc}
    \includegraphics[height = 0.18\linewidth]{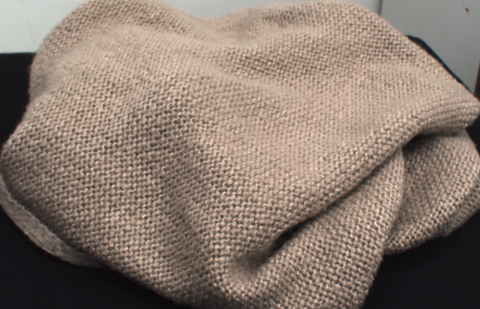} \hspace*{-3.3em}\includegraphics[height = 0.08\linewidth]{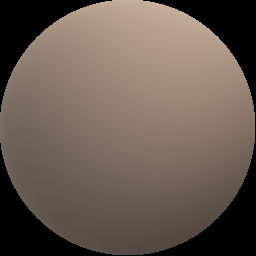}  &\, 
    \includegraphics[height = 0.18\linewidth]{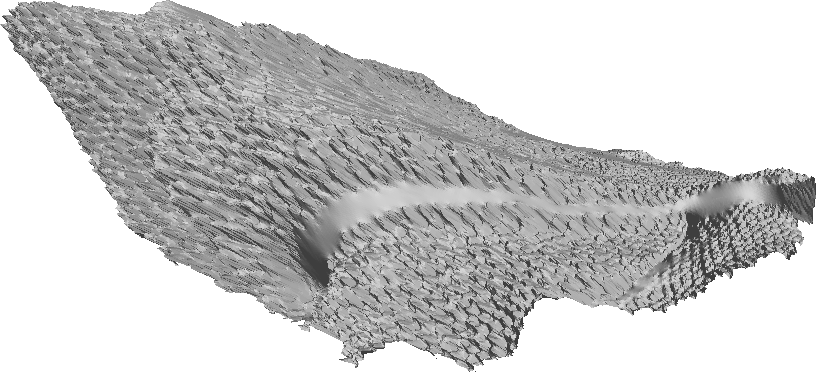} &\, 
    \includegraphics[height = 0.18\linewidth]{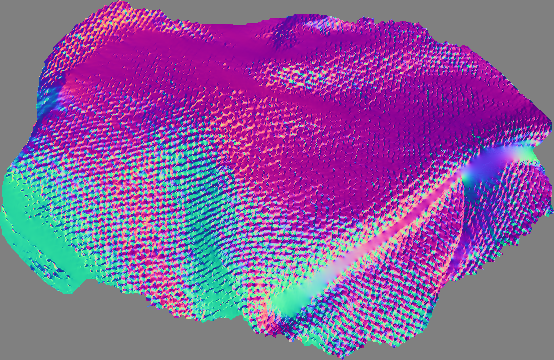} \\ 
    {\small Input real image} & \multicolumn{2}{c}{\small SFS result (no regularization)} \\ 
    {\small with illumination~\cite{Han2013}} & \multicolumn{2}{c}{($(\lambda,\mu,\nu) = (1,0,0)$)}\\[.5em]
  \end{tabular}
  \begin{tabular}{cccc}
    \includegraphics[height = 0.14\linewidth]{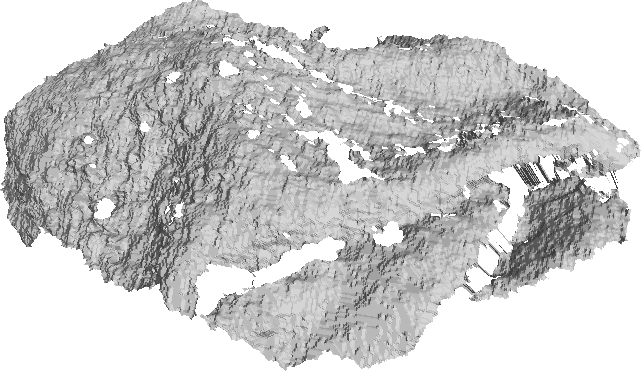} &\, 
    \includegraphics[height = 0.14\linewidth]{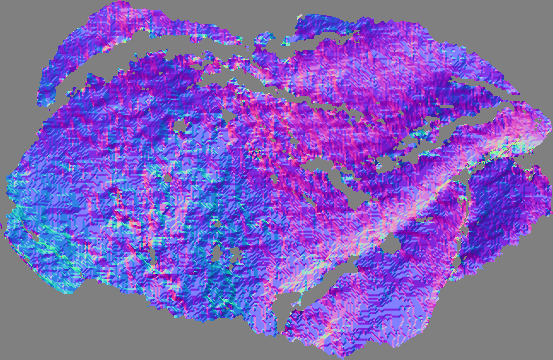} &\, 
    \includegraphics[height = 0.14\linewidth]{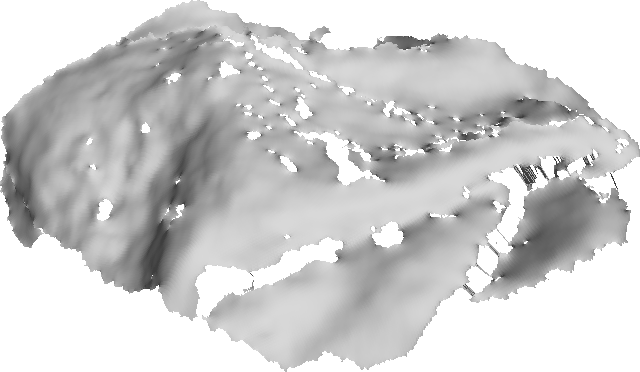} &\, 
    \includegraphics[height = 0.14\linewidth]{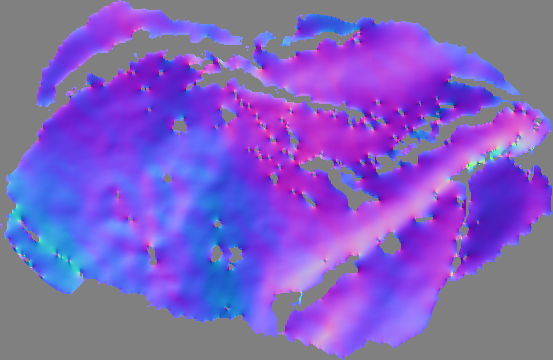} \\ 
   \multicolumn{2}{c}{\small Noisy input shape and normals~\cite{Han2013}} & \multicolumn{2}{c}{\small Minimal surface denoising (no SFS)}\\
   & & \multicolumn{2}{c}{ ($(\lambda,\mu,\nu) = (0,1,5.10^{-5})$) }\\[.5em]      
  \end{tabular}
  \begin{tabular}{cc}
    \includegraphics[height = 0.28\linewidth]{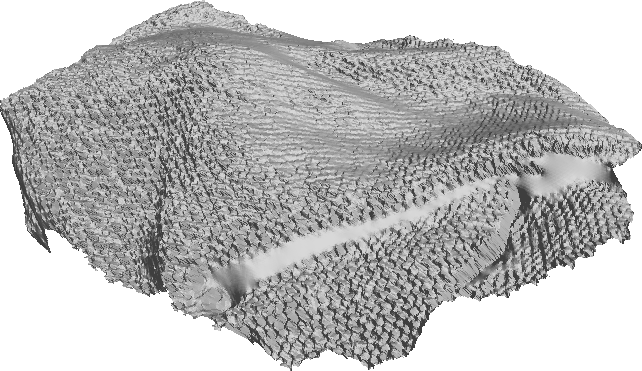} &\, 
    \includegraphics[height = 0.28\linewidth]{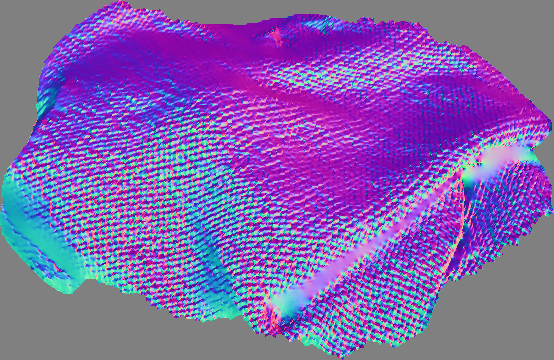} \\
    \multicolumn{2}{c}{\small SFS-based denoising and completion ($(\lambda,\mu,\nu) = (1,1,5.10^{-5})$)}
  \end{tabular}
  \caption{We propose the generic variational framework~\eqref{eq:13} for shape-from-shading (SFS) under natural illumination (top row). It is able to estimate a smooth surface (out of infinitely many), which almost exactly solves the generic SFS model~\eqref{eq:2}. To disambiguate SFS and improve robustness, prior surface knowledge (middle row, left) and minimal surface regularization (middle row, right) can be further included in the variational framework. These building blocks can be put together for shading-aware joint depth denoising, refinement and completion (bottom row).}
  \label{fig:teaser}
\end{figure}

\section{Image Formation Model and Related Works }
\label{sec:2}

In the following, a 3D-frame $(Oxyz)$ is attached to the camera, $O$ being the optical center and the axis $Oz$ coinciding with the optical axis, such that $z$ is oriented towards the scene. We denote $I:\,\Omega \subset \RR^2 \to \RR^C,\,(x,y) \mapsto I(x,y) = \left[I^1(x,y),\dots,I^C(x,y)\right]^\top$ a greylevel ($C=1$) or multi-channel ($C>1$) image of a surface, where $\Omega$ represents a ``mask'' of the object being pictured. We assume that the surface is Lambertian, so its reflectance is completely characterized by the albedo $\rho$. We further consider a second-order spherical harmonic model for the lighting vector $\bl$. To deal with the spectral dependencies of reflectance and lighting, we assume a general model where both $\rho$ and $\bl$ are channel-dependent. The albedo is thus a function $\rho:\,\Omega \to \RR^C,\,(x,y) \mapsto \rho(x,y) = \left[\rho^1(x,y),\dots,\rho^C(x,y)\right]^\top$, and the lighting in each channel $c \in \{1,\dots,C\}$ is represented as a vector $\bl^c = \left[l^c_1,l^c_2,l^c_3,l^c_4,l^c_5,l^c_6,l^c_7,l^c_8,l^c_9\right]^\top \in \RR^9$. Eventually, let $\bn:\,\Omega \to \mathbb{S}^2 \subset \mathbb{R}^3,\,(x,y) \mapsto \bn(x,y) = \left[n_1(x,y),n_2(x,y),n_3(x,y)\right]^\top$ be the field of unit-length outward  normals to the surface. The image formation model is then written as the following extension of a well-known model~\cite{Basri2003}:
\begin{equation}
  I^c(x,y) = \rho^c(x,y) \, \bl^c
  \cdot 
  \begin{bmatrix}
    \mathbf{n}(x,y) \\
    1 \\
    n_1(x,y) n_2(x,y) \\ n_1(x,y) n_3(x,y) \\ n_2(x,y) n_3(x,y) \\ {n_1(x,y)}^2 - {n_2(x,y)}^2 \\ 3 {n_3(x,y)}^2 -1
  \end{bmatrix},~(x,y) \in \Omega,~c \in \{1,\dots,C\}.
  \label{eq:1} 
\end{equation}

However, let us remark that the writing~\eqref{eq:1}, where both the reflectance and the lighting are channel-dependent, is abusive. Since the camera response function is also channel-dependent, this model is indeed justified only for white surfaces ($\rho^c = \rho,~\forall c \in \{1,\dots,C\}$) under colored lighting, or for colored surfaces under white lighting ($\bl^c = \bl,~\forall c \in \{1,\dots,C\}$). See, for instance,~\cite{Queau_2016_CVPR} for some discussion. In the following, we still consider the general model~\eqref{eq:1}, with a view in designing a generic SFS solver handling both situations. However, in the experiments we will only consider white surfaces.

In SFS, both the reflectance values $\{\rho^c\}_{c \in \{1,\dots,C\}}$ and the lighting vectors $\{\bl^c\}_{c \in \{1,\dots,C\}}$ are assumed to be known. The goal is to recover the object shape, represented in~\eqref{eq:1} by the normal field $\bn$. Each unit-length normal vector $\bn(x,y)$ has two degrees of freedom, thus each Equation~\eqref{eq:1}, $(x,y) \in \Omega$, $c \in \{1,\dots,C\}$, is a nonlinear equation with two unknowns. If $C=1$, it is impossible to solve such an equation locally: all these equations must be solved together, by coupling the surface normals in order to ensure, for instance, surface smoothness. When $C>1$ and the lighting vectors are non-coplanar, ambiguities theoretically disappear~\cite{Johnson2011}. However, under natural illumination these vectors are close to being collinear, and thus locally solving~\eqref{eq:1} is numerically challenging (bad conditioning). Again, a global solution should be preferred but this time, for robustness reasons.  

There is a large amount of literature on numerical SFS, in the specific case where $C=1$ and lighting is directional ($l^c_4 = \dots = l^c_9 = 0)$, see for instance~\cite{Durou2008}. However, few SFS methods deal with more general spherical harmonic lighting. First-order harmonics have been considered in~\cite{Huang2011,Or-El2015}, but they only capture up to $90\%$ of natural illumination, while this rate is over $99\%$ using second-order harmonics~\cite{Frolova2004}. The latter have been used in~\cite{Barron2015}, where the challenging problem of shape, illumination and reflectance from shading (SIRFS) is tackled (this method is also applicable to SFS if albedo and lighting are fixed). However, all these works heavily rely on multi-scale or regularization mechanisms, and not only for disambiguation or for handling noise. For instance, SIRFS ``fails badly''~\cite{Barron2015} without a multi-scale strategy, and the method of~\cite{Or-El2015} becomes unstable without depth regularization (see Figure~\ref{fig:5}). Although regularization mechanisms somewhat circumvent such numerical instabilities in practice, an ideal numerical solver would rely on regularization only for disambiguation and for handling noise, not for enforcing numerical stability. In order to design such a solver, a variational approach based on PDEs may be worthwile. In the next section, we thus rewrite~\eqref{eq:1} as a nonlinear PDE.

\begin{figure}[!ht]
\centering
\begin{tabular}{cc}
\includegraphics[height=0.3\linewidth]{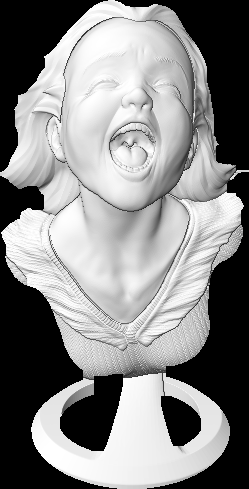}
\hspace*{-0.0cm}\includegraphics[height=0.05\linewidth]{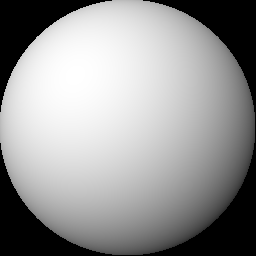}  \quad & \quad
\includegraphics[height=0.3\linewidth,trim = 10em 6em 10em 6em,clip]{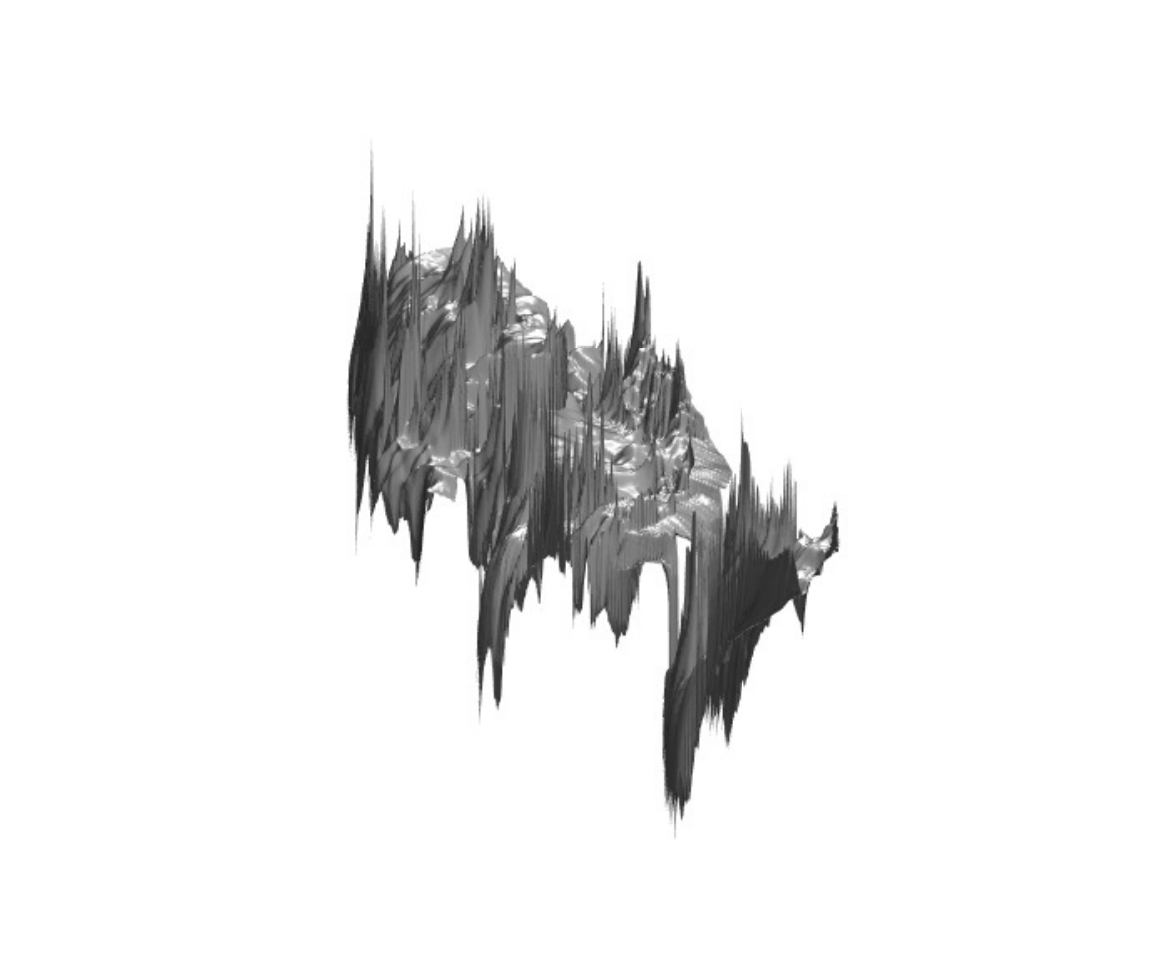} \\
{\small Input synthetic image and illumination} \quad & \quad {\small Fixed point~\cite{Or-El2015} without regularization} \\
\includegraphics[height=0.3\linewidth,trim = 8em 9em 8em 5em,clip]{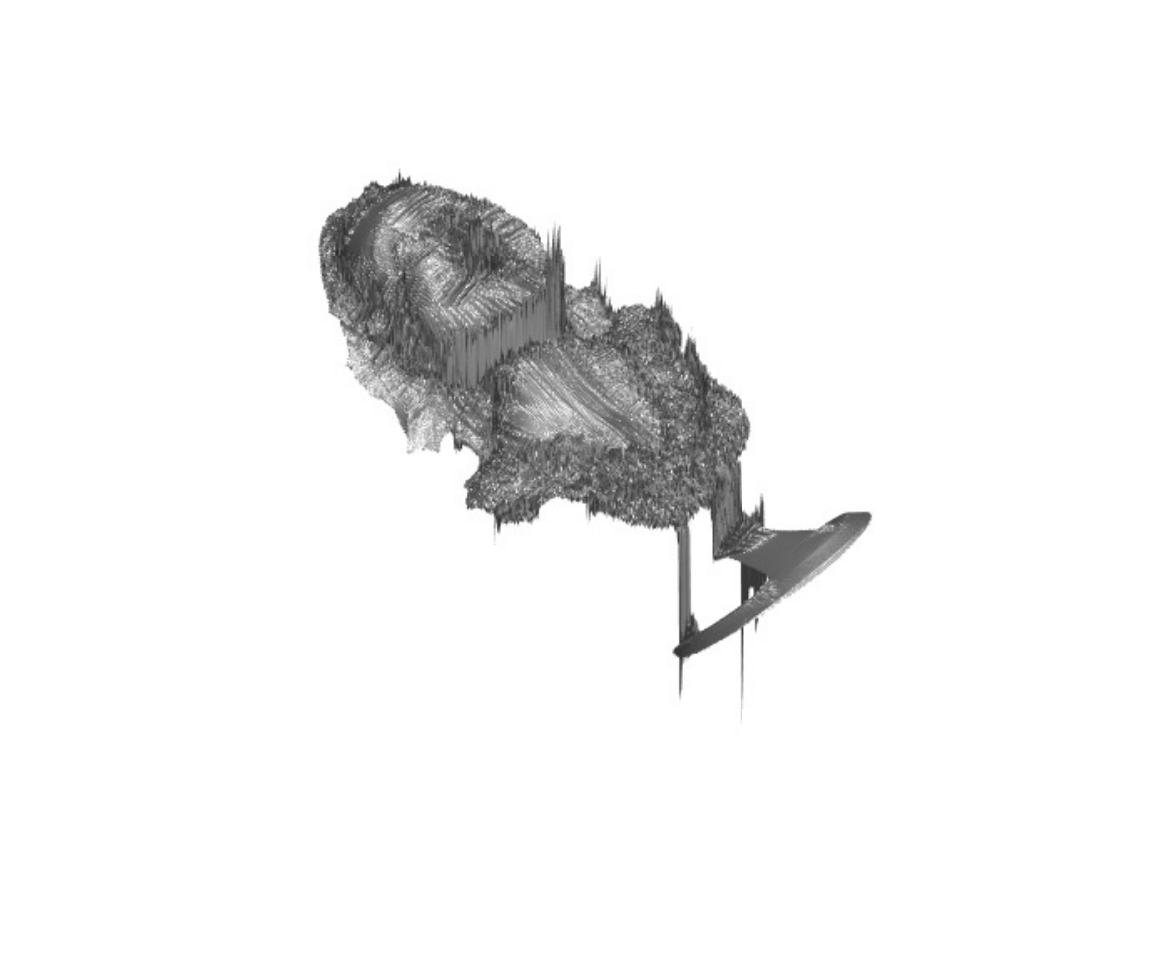} \quad&\quad 
\includegraphics[height=0.3\linewidth,trim = 7em 6em 7em 4.5em,clip]{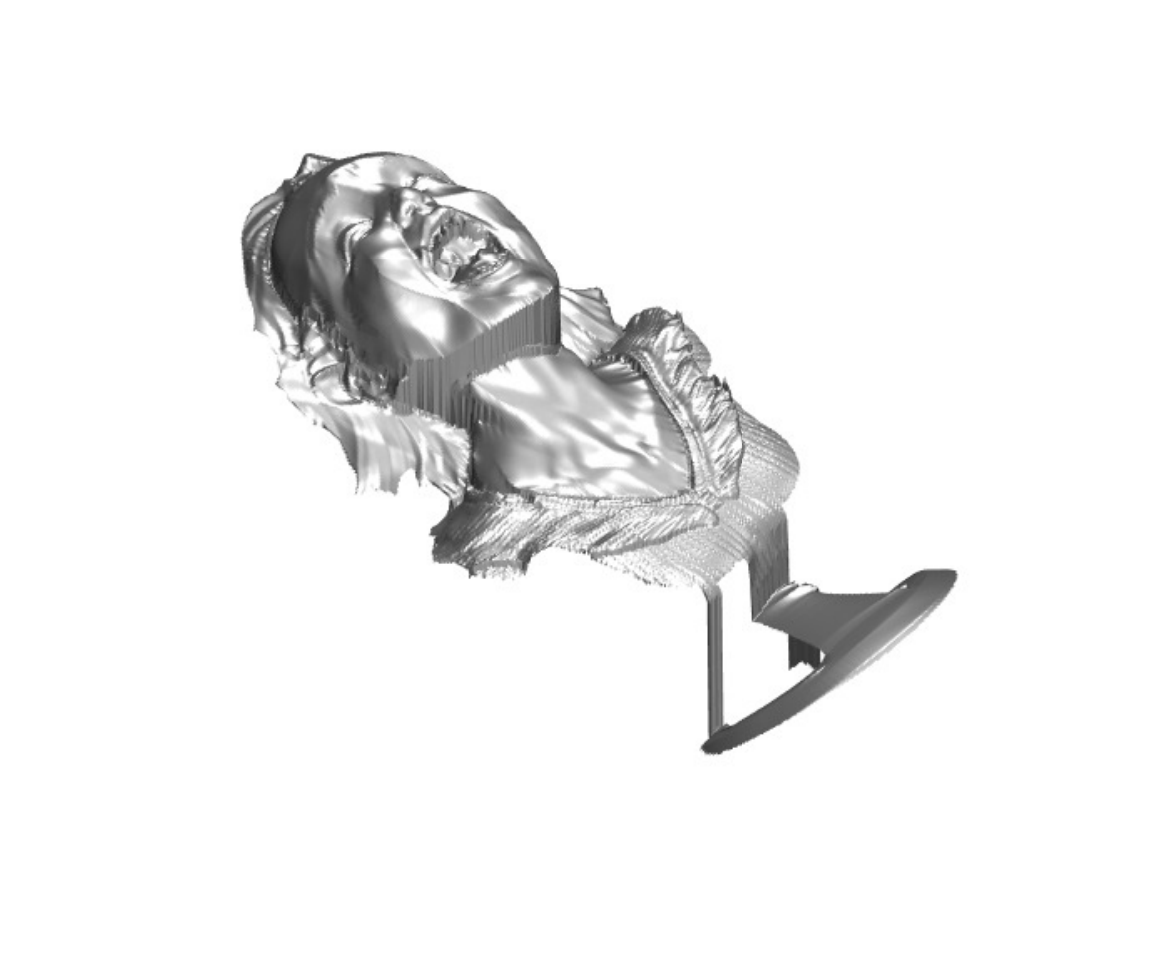} \\
 {\small Single-scale SIRFS~\cite{Barron2015}}\quad &\quad  {\small Proposed (without regularization)} 
%\\ {\small (a)} & {\small (b)}
\end{tabular}
% \begin{tabular}{cccc}
%{\small (c)} & {\small (d)} \\
%\\ {\small (e)} & {\small (f)}
% \end{tabular}
\caption{Greylevel shape-from-shading using first-order spherical harmonics. Linearization strategies such as the fixed point one used in~\cite{Or-El2015} fail if regularization is not employed. Similar issues arise in SIRFS~\cite{Barron2015} when the multi-scale approach is not used. Our SFS method can use regularization for disambiguation and for improving robustness, but it remains stable even without. In these three experiments, the same initial shape was used (the ``Realistic initialization'' of Figure~\ref{fig:experiments}). }
\label{fig:5}
\end{figure}

\section{A Generic PDE-based Model for Shape-from-shading}
\label{sec:3}

We assume hereafter that lighting and albedo are known (in our experiments, the albedo is assumed uniformly white and colored lighting is estimated from a gross surface approximation). These assumptions are usual in the SFS literature. They could be relaxed by simultaneously estimating shape, illumination and reflectance~\cite{Barron2015}, but we leave this as future work and focus only on shape estimation. This is the most challenging part anyways, since~\eqref{eq:1} is linear in the lighting and the albedo, but is generally nonlinear in the normal.

In order to comply with the discussion above, Equation~\eqref{eq:1} should be solved \emph{globally} over the entire domain $\Omega$. To this end, we do not estimate the normals but rather the underlying depth map, through a PDE-based approach~\cite{Lions1993}. This has the advantage of implicitly enforcing smoothness (almost everywhere) without requiring any regularization term (regularization will be introduced in Section~\ref{sec:4}, but only for the sake of disambiguation and robustness against noise). We show in this section the following result:

\begin{proposition}
Under both orthographic and perspective projections, the image formation model~\eqref{eq:1} can be rewritten as the following nonlinear PDE in $z$:
\begin{equation}
  \mathbf{a}^c_{(\nabla z)} \cdot \nabla z
  +b^c_{(\nabla z)} = I^c\quad \text{~over~} \Omega,~c \in \{1,\dots,C\}
  \label{eq:2}
\end{equation}
with $z:\,\Omega \to \RR$ a map characterizing the shape, $\nabla z:\,\Omega \to \RR^2$ its gradient, and where $\mathbf{a}^c_{(\nabla z)}:\, \Omega \to \RR^2$ and $b^c_{(\nabla z)}:\,\Omega \to \RR$ are a vector field and a scalar field, respectively, which depend in a nonlinear way on $\nabla z$.\newline{}
\end{proposition}

\begin{proof}

The 3D-shape can be represented as a patch over the image domain, which associates each pixel $(x,y) \in \Omega$ to its conjugate 3D-point $\bx(x,y)$ on the surface: 
\begin{equation}
\begin{array}{rcl}
  \mathbf{x}:\,& \Omega & \to \RR^3 \\
 &  (x,y) & \mapsto 
 \begin{cases}
   \left[x,y,\tilde{z}(x,y)\right]^\top & \text{under orthographic projection,} \\
   \tilde{z}(x,y)\left[\frac{x-x_0}{\tilde{f}},\frac{y-y_0}{\tilde{f}},1\right]^\top & \text{under perspective projection,}
 \end{cases} 
\end{array}
  \label{eq:3}
\end{equation}
with $\tilde{z}$ the \emph{depth map}, $\tilde{f} >0$ the \textit{focal length}, and $(x_0,y_0) \in \Omega$ the coordinates of the \textit{principal point} in the image plane. 

Using this parameterization, the normal to the surface in a surface point $\mathbf{x}(x,y)$ is the unit-length, outgoing vector proportional to the cross product $\mathbf{x}_x(x,y) \times  \mathbf{x}_y(x,y)$, where $\mathbf{x}_x$ (resp. $\mathbf{x}_y$) is the partial derivative of $\mathbf{x}$ along the $x$ (resp. $y$)-direction. After a bit of algebra, the following formula is obtained, which relates the normal field to the depth map:
\begin{equation}
\begin{array}{rcl}
  \mathbf{n}:\,& \Omega & \to \mathbb{S}^2 \subset \RR^3 \\
 &  (x,y) & \mapsto \dfrac{1}{d_{(\nabla z)}(x,y)} \begin{bmatrix} {f} \, \nabla z(x,y) \\ -1 - [\tilde{x},\tilde{y}]^\top \cdot \nabla z(x,y) \end{bmatrix}, 
\end{array}
  \label{eq:4}
\end{equation}
where 
\begin{equation}
({z},{f},\tilde{x},\tilde{y}) = 
\begin{cases}
  (\tilde{z},1,0,0) & \text{under orthographic projection},\\
  (\log \tilde{z},\tilde{f},x-x_0,y-y_0) & \text{under perspective projection},  
\end{cases}
  \label{eq:5}
\end{equation}
and where the map $d_{(\nabla z)}$ ensures the unit-length constraint on $\mathbf{n}$:
\begin{equation}
\begin{array}{rcl}
  d_{(\nabla z)}:\,& \Omega & \to \RR \\
  & (x,y) & \mapsto \sqrt{f^2 \| \nabla z(x,y) \|^2 + \left( 1 + \left[\tilde{x},\tilde{y}\right]^\top \cdot \nabla z(x,y)  \right)^2}.
\end{array}
\label{eq:6}
\end{equation}
Note that $\| d_{(\nabla z)} \|_{\ell^1(\Omega)}$ is the total area of the surface, which will be used in Section~\ref{sec:4} for designing a regularization term.

By plugging~\eqref{eq:4} into~\eqref{eq:1}, we obtain the nonlinear PDE~\eqref{eq:2}, if we denote: 
\begin{align}
  & \begin{array}{rcl}
    \ba^c_{(\nabla z)}:\,& \Omega & \to \RR^2 \\
    & (x,y) & \mapsto 
    \dfrac{\rho^c(x,y)}{d_{(\nabla z)}(x,y)} \begin{bmatrix}
 f\,l^c_1-\tilde{x}\,l^c_3 \\
 f\,l^c_2-\tilde{y}\,l^c_3 
    \end{bmatrix},
  \end{array}
    \label{eq:7} \\
&  \begin{array}{rcl}
    b^c_{(\nabla z)}:\,& \Omega & \to \RR \\
    & (x,y) & \mapsto \rho^c \, \begin{bmatrix} l^c_3 \\ l^c_4 \\ l^c_5 \\ l^c_6 \\ l^c_7 \\l^c_8 \\ l^c_9 \end{bmatrix} \cdot \begin{bmatrix}
  \frac{-1}{d_{(\nabla z)}(x,y)} \\
  1 \\
  \frac{f^2 z_x(x,y) z_y(x,y)}{\left(d_{(\nabla z)}(x,y)\right)^2} \\
  \frac{f z_x(x,y)\left(-1-(\tilde{x},\tilde{y}) \cdot \nabla z(x,y) \right)}{\left(d_{(\nabla z)}(x,y)\right)^2} \\
  \frac{f z_y(x,y)\left(-1-(\tilde{x},\tilde{y}) \cdot \nabla z(x,y) \right)}{\left(d_{(\nabla z)}(x,y)\right)^2} \\  
  \frac{f^2\left({z_x}(x,y)^2-{z_y}(x,y)^2\right)}{\left(d_{(\nabla z)}(x,y)\right)^2} \\  
  \frac{3\left(-1-(\tilde{x},\tilde{y}) \cdot \nabla z(x,y)\right)^2}{\left(d_{(\nabla z)}(x,y)\right)^2}-1 
  \end{bmatrix}.
  \end{array}
    \label{eq:8}
\end{align}
 \qed 
\end{proof}

When $C=1$, the camera is orthographic and the lighting is directional and frontal (\ie, $l_3 < 0$ is the only non-zero lighting component), then~\eqref{eq:2} becomes the \emph{eikonal equation} $\frac{\rho |l_3|}{\sqrt{1+\|\nabla z\|^2}} = I$. Efficient numerical methods for solving this nonlinear PDE have been suggested, using for instance semi-Lagrangian schemes~\cite{Cristiani2007}. Such techniques can also handle perspective camera projection and$/$or nearby point light source illumination~\cite{Breuss2012}. Still, existing PDE-based methods require a boundary condition, or at least a state constraint, which is rarely available in practice. In addition, the more general PDE-based model~\eqref{eq:2}, which handles both orthographic or perspective camera, directional or second-order spherical harmonic lighting, and greylevel or multi-channel images, has not been tackled so far. A variational solution to this generic SFS problem, which is inspired by the classical method of Horn and Brooks~\cite{Horn1986}, is presented in the next section.

\section{Variational Formulation and Optimization}
\label{sec:4}

The $C$ PDEs in~\eqref{eq:2} are in general incompatible due to noise. Thus, an approximate solution must be sought. If we assume that the image formation model~\eqref{eq:1} is satisfied up to an additive, zero-mean and homoskedastic, Gaussian noise, then the maximum likelihood solution is attained by estimating the depth map $z$ which minimizes the following least-squares cost function:
\begin{equation}
  \mathcal{E}(\nabla z;I) =  \displaystyle\sum_{c=1}^C \left\| \ba^c_{(\nabla z)} \cdot \nabla z + b^c_{(\nabla z)} - I^c\right\|^2_{\ell^2(\Omega)}.
  \label{eq:9}
\end{equation}

%In recent works on shading-based refinement~\cite{Or-El2015}, it is suggested to minimize a cost function similar to~\eqref{eq:9} iteratively, by freezing the nonlinear fields $\ba^c$ and $b^c$ at each iteration. This strategy must be avoided. For instance, it cannot handle the simplest case of orthographic projection and directional, frontal lighting: this yields $\ba^c \equiv {\bm 0}$ according to~\eqref{eq:7}, and thus~\eqref{eq:9} does not even depend on the unknown depth $z$ if $b^c$ is freezed. Even in less trivial cases, Figure~\ref{fig:5} shows that this strategy is unstable, which explains why regularization is employed in~\cite{Or-El2015}. We also resort to regularization, but only for the sake of disambiguating SFS and handling noise \ie, our proposal yields a stable solution even in the absence of regularization (see Figure~\ref{fig:5}). In this work, we consider two types of regularization: one which represents prior knowledge of the surface, and another one which ensures its smoothness.

In recent works on shading-based refinement~\cite{Or-El2015}, it is suggested to minimize a cost function similar to~\eqref{eq:9} iteratively, by freezing the nonlinear fields $\ba^c$ and $b^c$ at each iteration. This strategy must be avoided. For instance, it cannot handle the simplest case of orthographic projection and directional, frontal lighting: this yields $\ba^c \equiv {\bm 0}$ according to~\eqref{eq:7}, and thus~\eqref{eq:9} does not even depend on the unknown depth $z$ if $b^c$ is freezed. Even in less trivial cases, Figure~\ref{fig:5} shows that this strategy is unstable, which explains why regularization is employed in~\cite{Or-El2015}. We will also resort to regularization later on, but only for the sake of disambiguating SFS and handling noise: our proposal yields a stable solution even in the absence of regularization (see Figure~\ref{fig:5}). Let us first sketch the proposed solver in the regularization-free case, and discuss its connection with Horn and Brooks' variational approach to SFS.

\subsection{Horn and Brooks' Method Revisited}

In~\cite{Horn1986}, Horn and Brooks introduce a variational approach for solving the eikonal SFS model, which is a special case of~\eqref{eq:2}. They promote a two-stages shape recovery method, which first estimates the gradient and then integrates it into a depth map. That is to say, the energy~\eqref{eq:9} is first minimized in terms of the gradient $\theta := \nabla z$, and then $\theta$ is integrated into a depth map $z$. 

Since local gradient estimation is ambiguous, they put forward the introduction of the so-called integrability constraint for the first stage. Indeed, $\theta$ is the gradient of a function $z$: it must be a conservative field. This implies that it should be irrotational (zero-curl condition). Introducing the divergence operator $\nabla \cdot$, the latter condition reads
\begin{equation}
\underbrace{
\left\|
\nabla \cdot \begin{bmatrix}
0 & 1 \\
-1 & 0
\end{bmatrix} \theta \right\|_{\ell^2(\Omega)}^2}_{:= \mathcal{I}(\theta)} = 0.
\label{eq:integ}
\end{equation}

In practice, they convert the hard-constraint~\eqref{eq:integ} into a regularization term, introducing two hyper-parameters $(\lambda,\mu)>(0,0)$ to balance the adequation to the images and the integrability of the estimated field:
\begin{equation}
  \widehat{\theta} = \underset{\theta:\,\Omega \to \RR^2}{\operatorname{argmin~}} \lambda \, \mathcal{E}(\theta;I)+\mu \, \mathcal{I}(\theta).
  \label{eq:HB1}
\end{equation}

After solving~\eqref{eq:HB1}, $\widehat{\theta}$ is integrated into a depth map $z$, by solving $\nabla z = \widehat{\theta}$. However, since integrability is not strictly enforced but only used as regularization, there is no guarantee for $\widehat{\theta}$ to be integrable. Therefore, Horn and Brooks recast the integration task as another variational problem (see~\cite{Queau_2017_JMIV} for an overview of this problem):
\begin{equation}
\min_{z:\,\Omega \to \RR} \left\| \nabla z - \widehat{\theta} \right\|_{\ell^2(\Omega)}^2.
\label{eq:HB2}
\end{equation}

This two-stages approach consisting in solving~\eqref{eq:HB1}, and then~\eqref{eq:HB2}, is however prone to bias propagation: any error during gradient estimation may have dramatic consequences for the integration stage. We argue that such a sequential approach is not necessary. Indeed, $\theta$ is conservative \textit{by construction}, and hence integrability should not even have to be invoked. We put forward an integrated approach, which infers shape clues from the image using local gradient estimation as in Horn and Brooks' method, but which explicitly constrains the gradient to be conservative. That is to say, we simultaneously estimate the depth map and its gradient, by turning the minimization of~\eqref{eq:9} into a constrained variational problem:
\begin{equation}
\begin{array}{c}
	\underset{\substack{\theta:\,\Omega \to \RR^2 \\ z:\, \Omega \to \RR }}{\operatorname{min\quad}} \mathcal{E}(\theta;I) \\
	\text{s.t.~} \nabla z = \theta
\end{array}
\label{eq:HB3}
\end{equation}

The variational problem~\eqref{eq:HB3} can be solved using an augmented Lagrangian approach. In comparison with Horn and Brooks' method, this avoids tuning the hyper-parameter $\mu$ in~\eqref{eq:HB1}, as well as bias propagation due to the two-stages approach. Besides, this approach is easily extended in order to handle regularization terms, if needed. In the next paragraph, we consider two types of regularization: one which represents prior knowledge of the surface, and another one which ensures its smoothness.

\subsection{Regularized Variational Model}

In some applications such as RGB-D sensing, or MVS, a depth map $z^0$ is available, which is usually noisy and incomplete but may represent a useful ``guide'' for SFS. We may thus consider the following prior term:
\begin{equation}
  \mathcal{P}(z;z^0) = \left\| z - z^0 \right\|_{\ell^2(\Omega^0)}^2,
    \label{eq:10}
\end{equation}
where $\Omega^0 \subseteq   \Omega \subset \RR^2$ is the image region for which prior information is available. 

In order not to interpret noise in the image as geometric artifacts, one may also want to improve robustness by explicitly including a smoothness term. However, standard total variation regularization, which is often considered in image processing, may tend to favor piecewise fronto-parallel surfaces and thus induce staircasing. We rather penalize the total area of the surface, which has recently been shown in~\cite{Graber2015} to be better suited for depth map regularization. To this end, let us remark that in differential geometry terms, the map $d_{(\nabla z)}$ defined in~\eqref{eq:6} is the square root of the determinant of the first fundamental form of function z (metric tensor). Its integral over $\Omega$ is exactly the area of the surface, and thus the following smoothness term may be considered:
\begin{equation}
  \mathcal{S}(\nabla z) = \left\| d_{(\nabla z)} \right\|_{\ell^1(\Omega)}. 
    \label{eq:11} 
\end{equation}

Putting altogether the pieces~\eqref{eq:9},~\eqref{eq:10} and~\eqref{eq:11}, and using the same change of variable $\theta:= \nabla z$ as in~\eqref{eq:HB3}, we obtain the following constrained variational approach to shape-from-shading:  
\begin{equation}
\begin{array}{l}
  \underset{\substack{\theta:\, \Omega \to\RR^2 \\ z:\,\Omega \to \RR}}{\min~}   \lambda \, \mathcal{E}(\theta;I) + \mu \, \mathcal{P}(z;z^0) + \nu \, \mathcal{S}(\theta)   \\
  \text{s.t.}~ \nabla z = \theta
\end{array}
\label{eq:13}
\end{equation} 
%\begin{equation}
%  \underset{z:\, \Omega \to \RR}{\min~} \lambda \, \mathcal{E}(\nabla z;I) + \mu \, \mathcal{P}(z;z^0) + \nu \, \mathcal{S}(\nabla z), 
%  \label{eq:12}
%\end{equation}
which is a regularized version of the pure shape-from-shading model~\eqref{eq:HB3} where $(\lambda,\mu,\nu) \geq (0,0,0)$ are user-defined parameters controlling the respective influence of each term. 

Let us remark that our variational model~\eqref{eq:13} yields a pure SFS model if $\mu = \nu = 0$, a depth denoising model similar to that in~\cite{Graber2015} if $\lambda =0$ and $\Omega^0 = \Omega$, and a shading-aware joint depth refinement and completion if $\lambda>0$, $\mu>0$ and $\Omega^0 \subsetneq  \Omega$.  

\subsection{Numerical Solution}

The change of variable $\theta:= \nabla z$ in~\eqref{eq:13} has a major advantage when it comes to numerical solving: it separates the difficulty  induced by the nonlinearity (shape-from-shading model and minimal surface prior) from that induced by the global nature of the problem (dependency upon the depth gradient). 

Optimization can then be carried out by alternating nonlinear, yet local gradient estimation and global, yet linear depth estimation. To this end, we make use of the ADMM procedure, a standard approach to constrained optimization which dates back to the 70s~\cite{Glowinski1975}. We refer the reader to~\cite{Boyd2011} for a recent overview of this method. 

The augmented Lagrangian functional associated to~\eqref{eq:13} is defined as follows:
\begin{equation}
\mathcal{L}_\beta(\theta,z,\Psi) =    \lambda \, \mathcal{E}(\theta;I) + \mu \, \mathcal{P}(z;z^0) + \nu \, \mathcal{S}(\theta)   + \langle \Psi, \nabla z - \theta \rangle + \frac{\beta}{2} \left\| \nabla z - \theta \right\|_2^2,
\end{equation}
with $\Psi:\,\Omega \to \RR^2$ the field of Lagrange multipliers, $\langle \cdotp \rangle$ the scalar product induced by $\|\cdotp\|_2$ over $\Omega$, and $\beta>0$.

 ADMM iterations are then written:
\begin{align} 
\theta^{(k+1)}  & = \underset{\theta}{\operatorname{argmin}~} \mathcal{L}_{\beta^{(k)}}(\theta,z^{(k)},\Psi^{(k)}), \label{eq:15} \\[-0.5em]
z^{(k+1)}  & = \underset{z}{\operatorname{argmin}~} \mathcal{L}_{\beta^{(k)}}(\theta^{(k+1)},z,\Psi^{(k)}), \label{eq:16} \\[-0.5em]
\Psi^{(k+1)} & = \!\Psi^{(k)} + \beta^{(k)} \left( \nabla z^{(k+1)} -\theta^{(k+1)}\right). 
\end{align}
where $\beta^{(k)}$ can be determined automatically~\cite{Boyd2011}.

Problem~\eqref{eq:15} is a pixelwise non-trivial optimization problem which is solved using a Newton method with an L-BFGS stepsize. As for~\eqref{eq:16}, it is discretized by first-order, forward finite differences with Neumann boundary condition. This yields a linear least-squares problem whose normal equations provide a symmetric, positive definite (semi-definite if $\mu = 0$) linear system. It is sparse, but too large to be solved directly: conjugate gradient iterations should be preferred.  In our experiments, the algorithm stops when the relative variation of the energy in~\eqref{eq:13} falls below $10^{-3}$.  

This ADMM algorithm can be interpreted as follows. During the $\theta$-update~\eqref{eq:15}, local estimation of the gradient (\ie, of the surface normals) is carried out based on SFS, while ensuring that the gradient map is smooth and  close to the gradient of the current depth map. Unlike in the fixed point approach~\cite{Or-El2015}, local surface orientation is inferred from the whole model~\eqref{eq:2}, and not only from its linear part. In practice, we observed that this yields a much more stable algorithm (see Figure~\ref{fig:5}). In the $z$-update~\eqref{eq:16}, these surface normals are integrated into a new depth map, which should stay close to the prior. 

Given the non-convexity of the shading term $\mathcal{E}$ and of the smoothness term~$\mathcal{S}$, convergence of the ADMM algorithm is not guaranteed. However, in practice we did not observe any particular convergence-related issue, so we conjecture that a convergence proof could eventually be provided, taking inspiration from the recent studies on non-convex ADMM~\cite{li2015global,hong2016convergence}. However, we leave this as future work and focus in this proof of concept work on sketching the approach and providing preliminary empirical results. 

The next section shows quantitatively the effectiveness of the proposed ADMM algorithm for solving SFS under natural illumination, and introduces qualitative results on real-world datasets.

\section{Experiments}
\label{sec:5}

\subsection{Quantitative Evaluation of the Proposed SFS Framework}

We first validate in Figure~\ref{fig:experiments} the ability of the proposed variational framework to solve SFS under natural illumination \ie, to solve~\eqref{eq:1}. Our approach is compared against SIRFS~\cite{Barron2015}, which is the only method for SFS under natural illumination whose code is freely available. For fair comparison, albedo and lighting estimations are disabled in SIRFS, and its multi-scale strategy is used, in order to avoid the artifacts shown in Figure~\ref{fig:5}. 

Since we only want to compare here the ability of both methods to explain a shaded image, our regularization terms are disabled ($\mu = \nu = 0$), as well as those from SIRFS. To quantify this ability, we measure the RMSE between the input images and the reprojected ones, as advised in~\cite{Durou2008}. 
% For completeness, the mean angular error (MAE) between the ground truth shape and the estimated ones is also provided, though it may not be very meaningful in view of the inherent ambiguities of SFS. 

\begin{figure*}
  % Top row: Ground truth, Image + SH, Image + SH, Image + SH
  \setlength{\tabcolsep}{0.0em} % for the horizontal padding
{\renewcommand{\arraystretch}{0.6}% for the vertical padding
\begin{tabular}{ccccc}
    \multicolumn{2}{c}{\includegraphics[width=0.18\linewidth,trim = 8em 6em 7em 4em,clip]{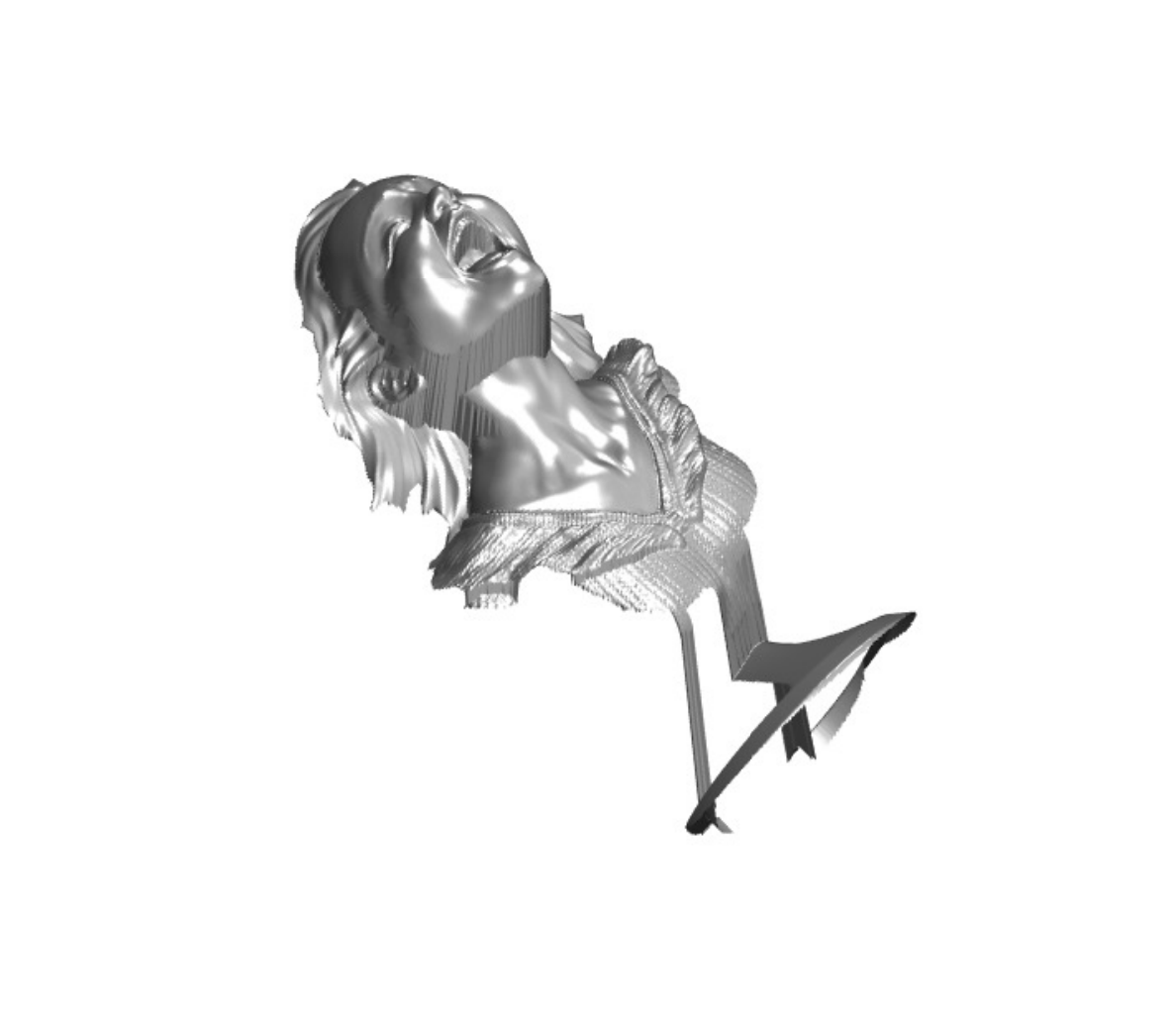}} &
    \includegraphics[height=0.18\linewidth]{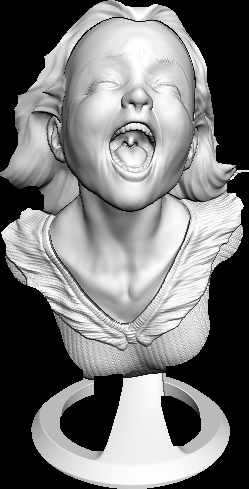}
    \includegraphics[height=0.07\linewidth]{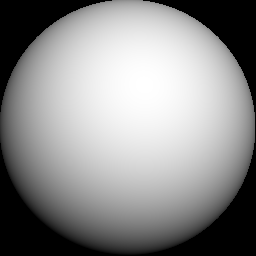} &
    \includegraphics[height=0.18\linewidth]{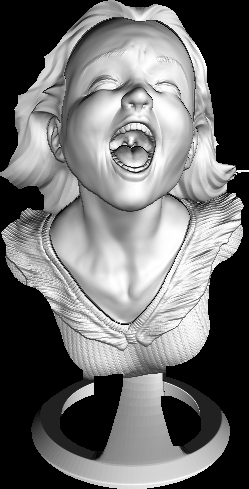} 
    \includegraphics[height=0.07\linewidth]{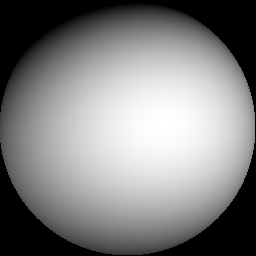} &  
    \includegraphics[height=0.18\linewidth]{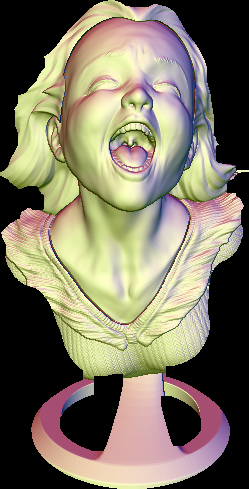} 
    \includegraphics[height=0.07\linewidth]{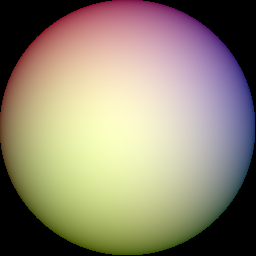} \\
 \multicolumn{2}{c}{ {\small Ground truth}} & {\small Greylevel,} & {\small Greylevel,} & {\small Colored,} \\
 \multicolumn{2}{c}{ ~} & {\small first-order} & {\small second-order} & {\small second-order} \\ 
 \multicolumn{2}{c}{ ~} & {\small lighting $\mathbf{l}_1$~\eqref{eq:18}} & {\small  lighting $\mathbf{l}_2$~\eqref{eq:19}} & {\small lighting $\mathbf{l}_3$~\eqref{eq:20}} \\[1em]  
%%%%%%%%%%%%%%%%%%%%%%%%%%%%%%%%%%%%%%%%%%%%%%%%%%%% 
  % Initial guess
  \multirow{3}{*}{
  \begin{tabular}{c}
  \includegraphics[width=0.15\linewidth,trim = 7em 6em 6em 4em,clip]{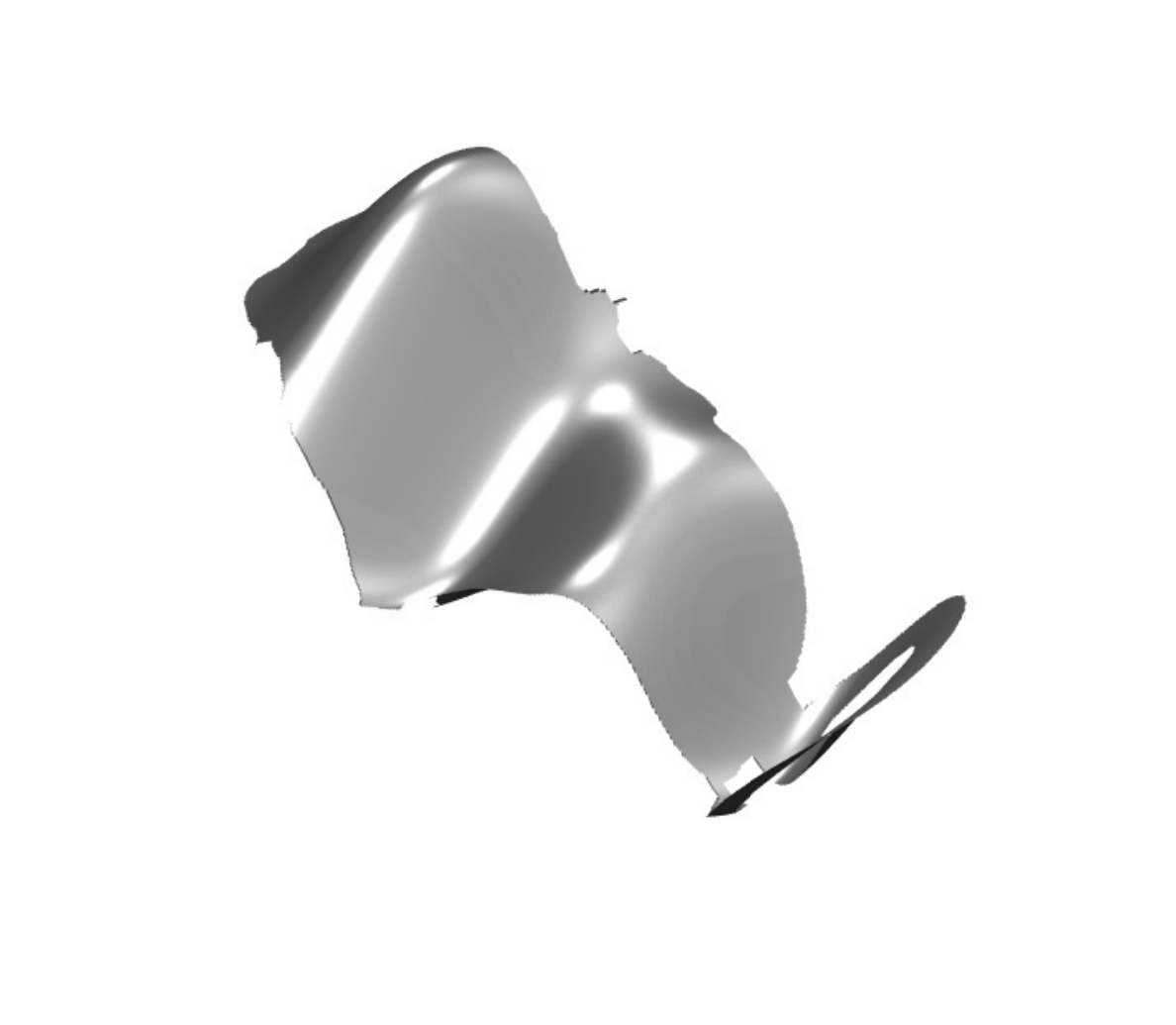} \\ Non-realistic \quad\\ initialization \quad
  \end{tabular}} &
  % Ours - Dir
 \parbox[t]{2.5mm}{\rotatebox[origin=c]{90}{Ours}} &
 \includegraphics[height=0.15\linewidth,trim = 8em 6em 7em 4em,clip]{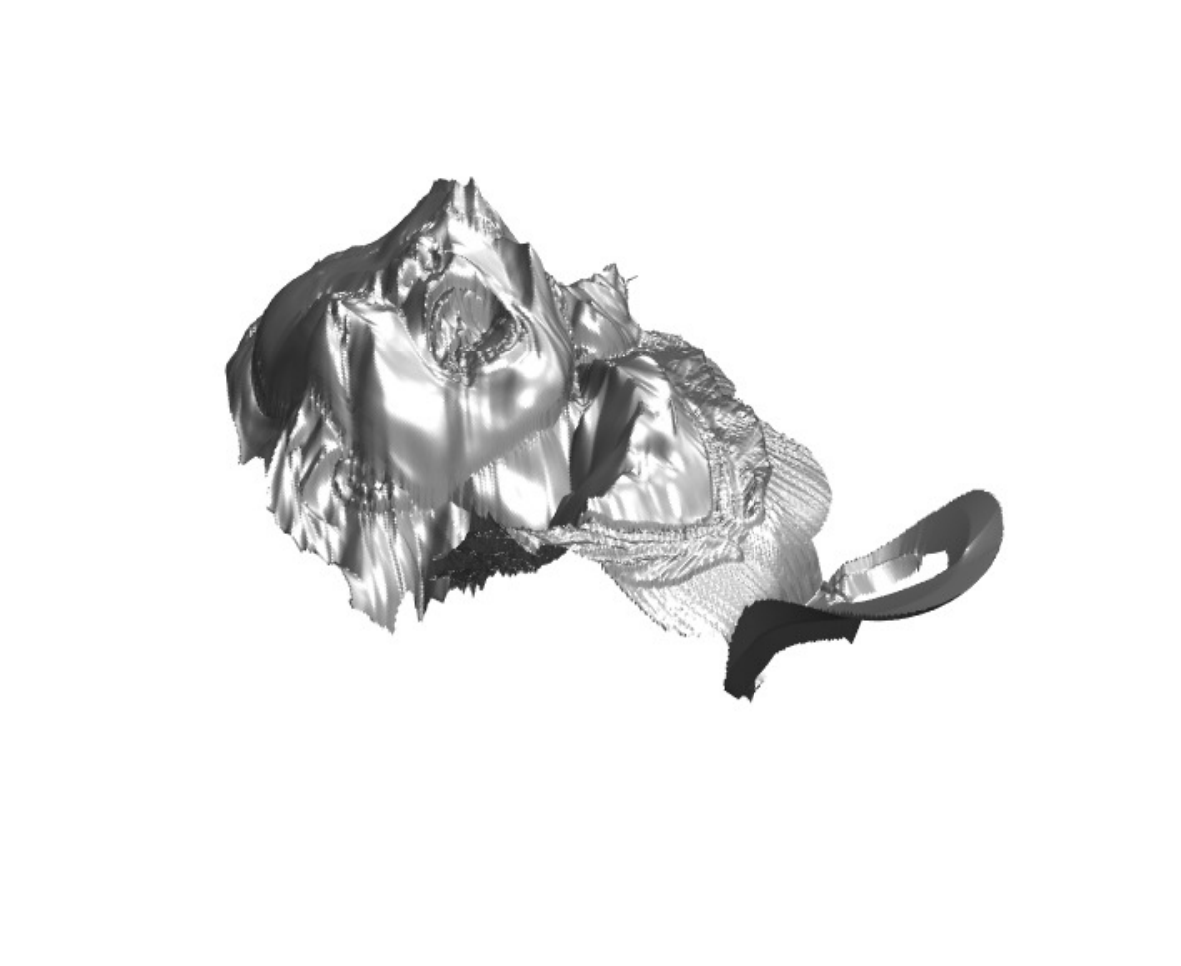}
 \includegraphics[height=0.15\linewidth]{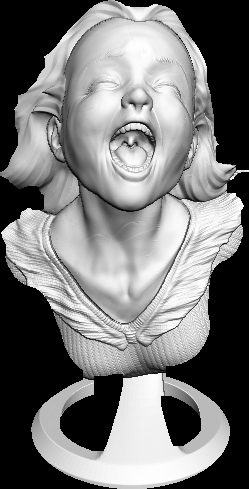} &
  % Ours - SH2
 \includegraphics[height=0.15\linewidth,trim = 6em 6em 5em 4em,clip]{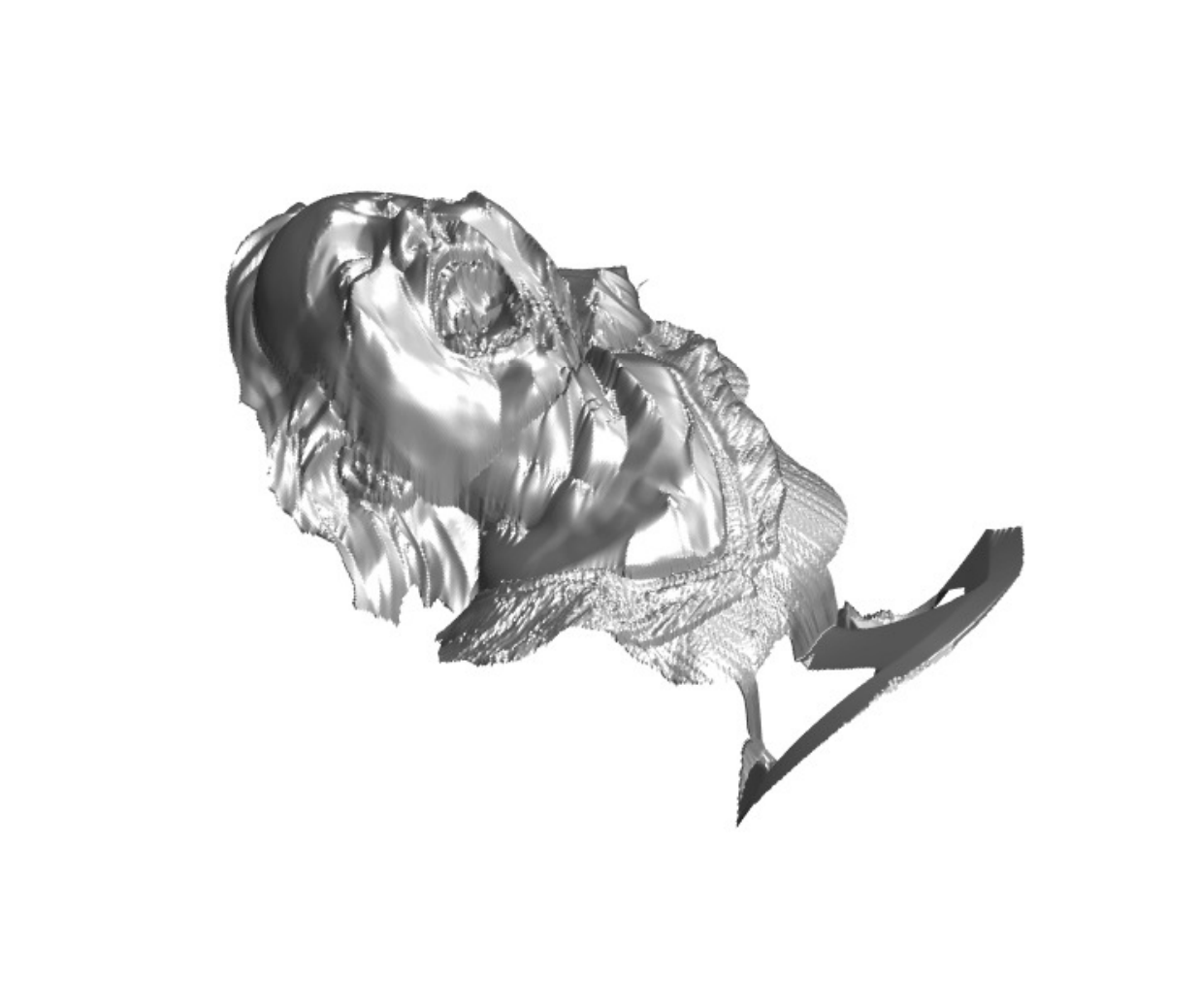} 
\includegraphics[height=0.15\linewidth]{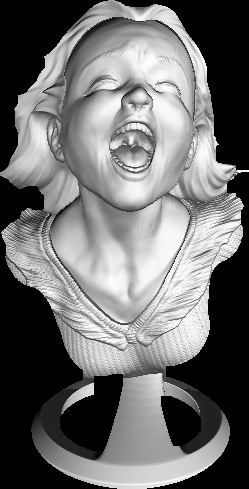} &
  % Ours - SH2 color
\includegraphics[height=0.15\linewidth,trim = 6em 6em 5em 4em,clip]{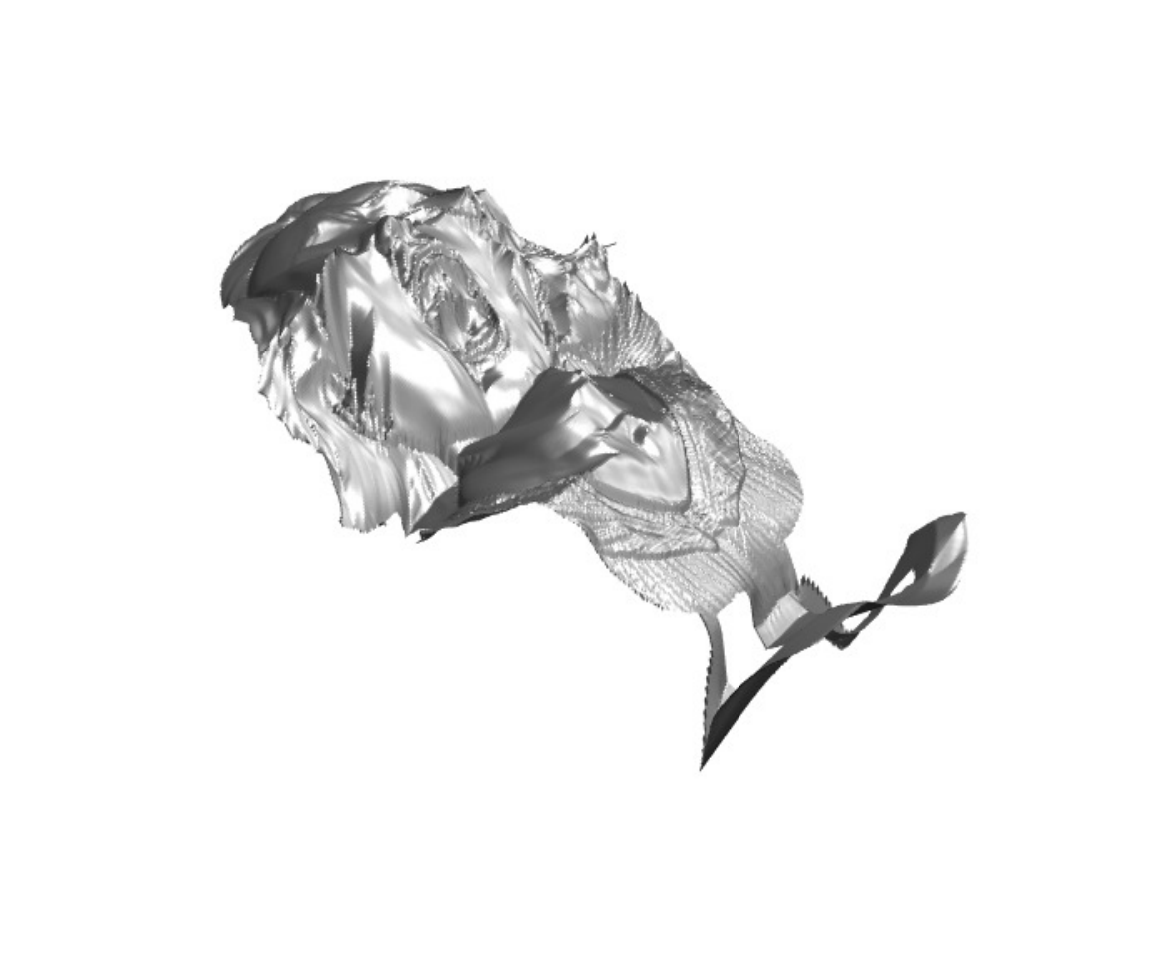} 
\includegraphics[height=0.15\linewidth]{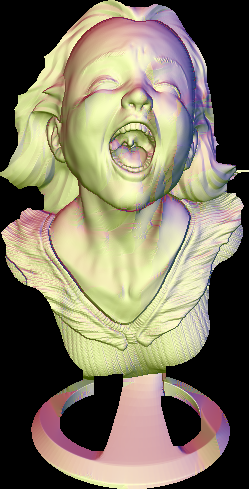} \\
 & & {\small RMSE $= \textbf{0.02}$} & {\small RMSE $= \textbf{0.03}$} & {\small RMSE $= \textbf{0.05}$} \\
%%%%%%%%%%%%%%%%%%
 & % {\small Non-realistic initial guess} & % Same for SIRFS
  % Ours - Dir
 \parbox[t]{2.5mm}{\rotatebox[origin=c]{90}{SIRFS}} &
 \includegraphics[height=0.15\linewidth,trim = 6em 6em 6em 4em,clip]{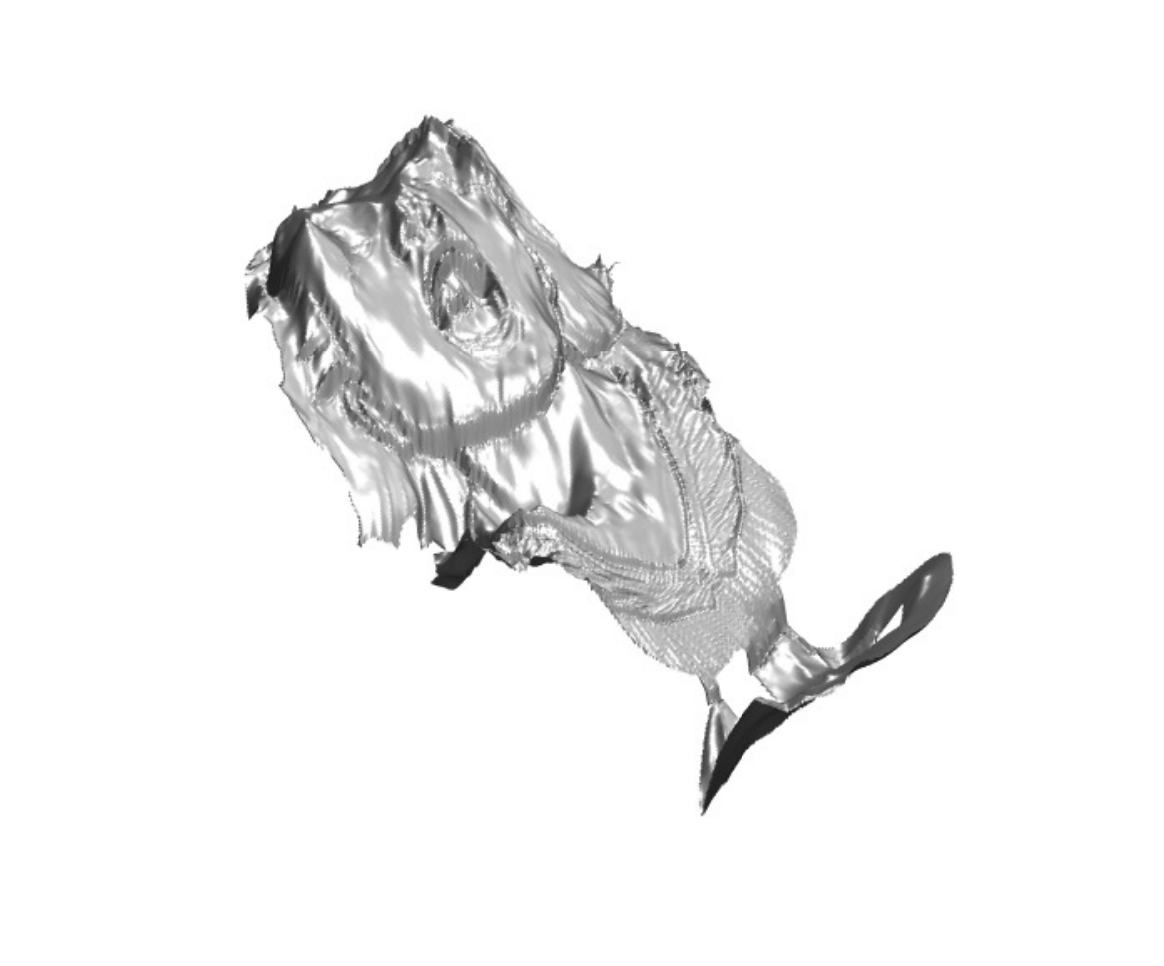}
 \includegraphics[height=0.15\linewidth]{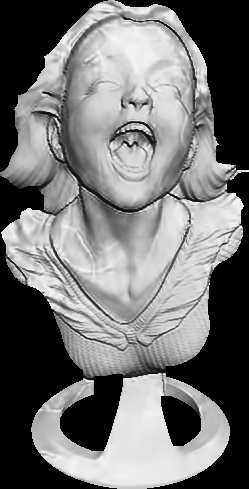} &
  % Ours - SH2
 \includegraphics[height=0.15\linewidth,trim = 6em 6em 6em 4em,clip]{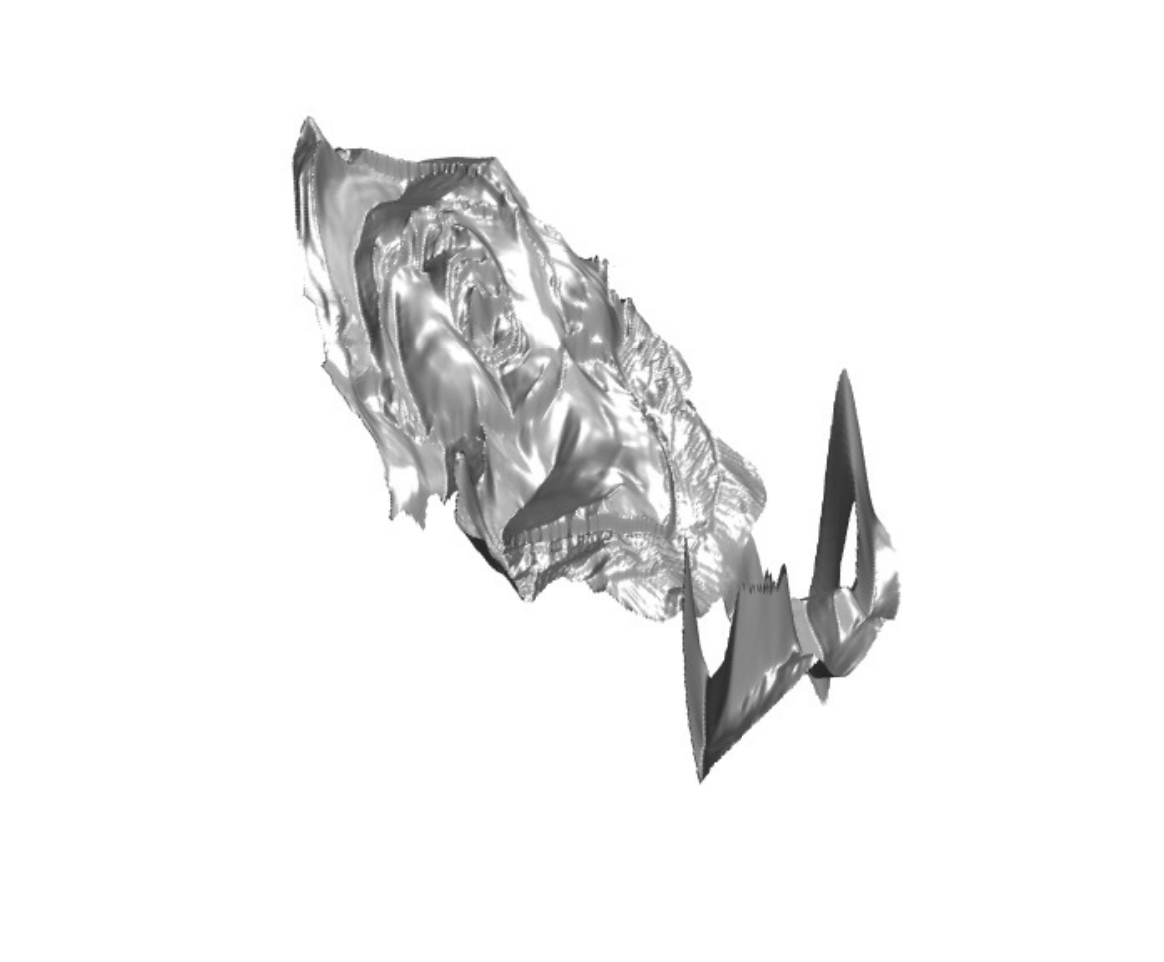} 
\includegraphics[height=0.15\linewidth]{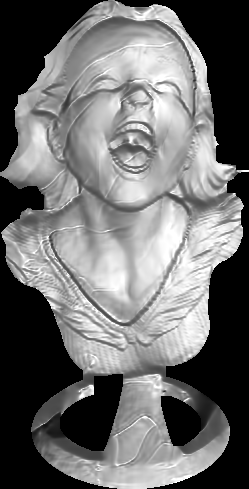} &
  % Ours - SH2 color
\includegraphics[height=0.15\linewidth,trim = 6em 6em 6em 4em,clip]{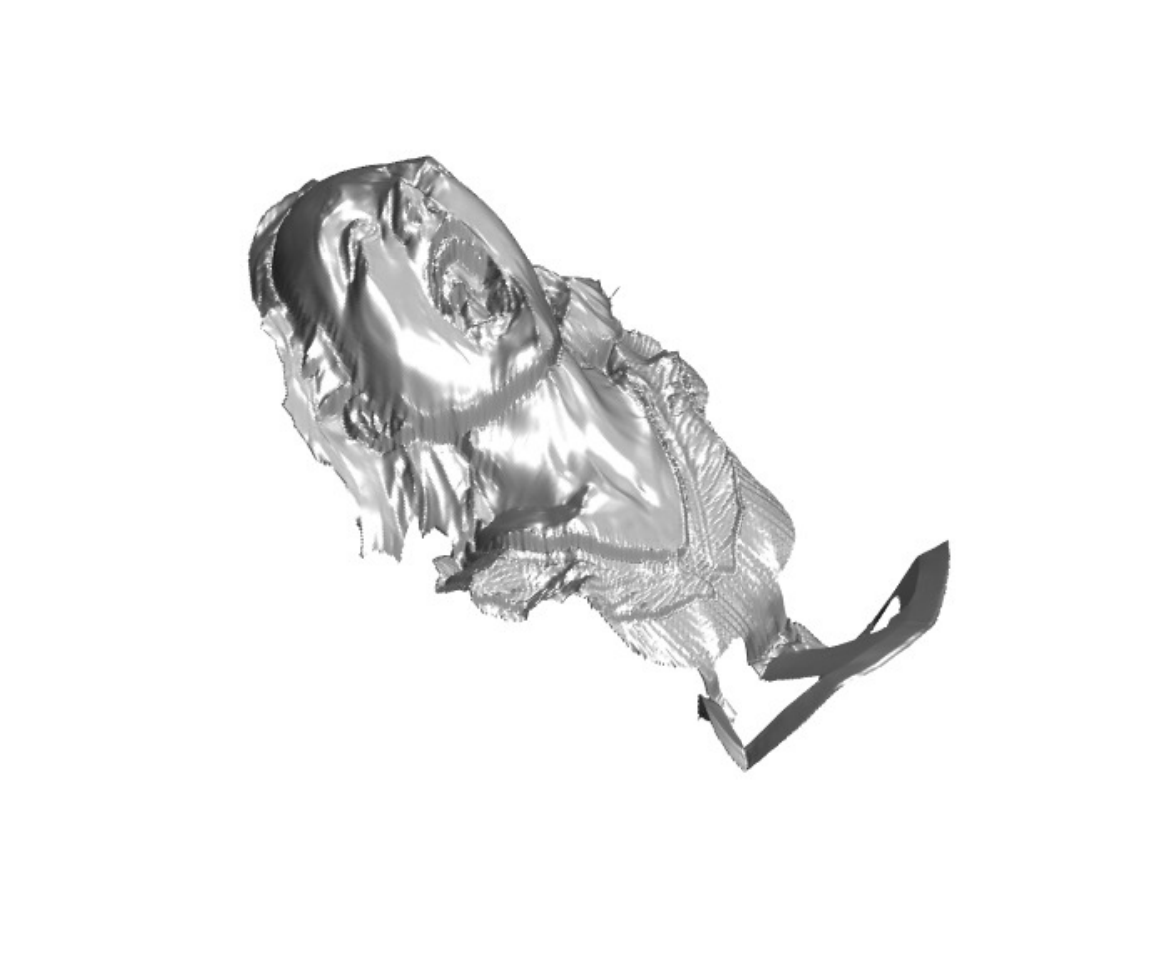} 
\includegraphics[height=0.15\linewidth]{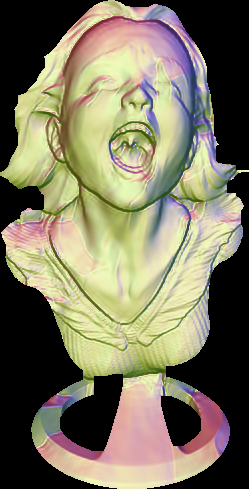} \\
 & & {\small RMSE $= 0.08$} & {\small RMSE $= 0.10$} & {\small RMSE $= 0.08$} \\[1em]
%%%%%%%%%%%%%%%%%%%%%%%%%%%%%%%%%%%%%%%%%%%%%%%%%%%%
 \multirow{2}{*}{
 \begin{tabular}{c}
  \includegraphics[width=0.15\linewidth,trim = 7em 6em 6em 4em,clip]{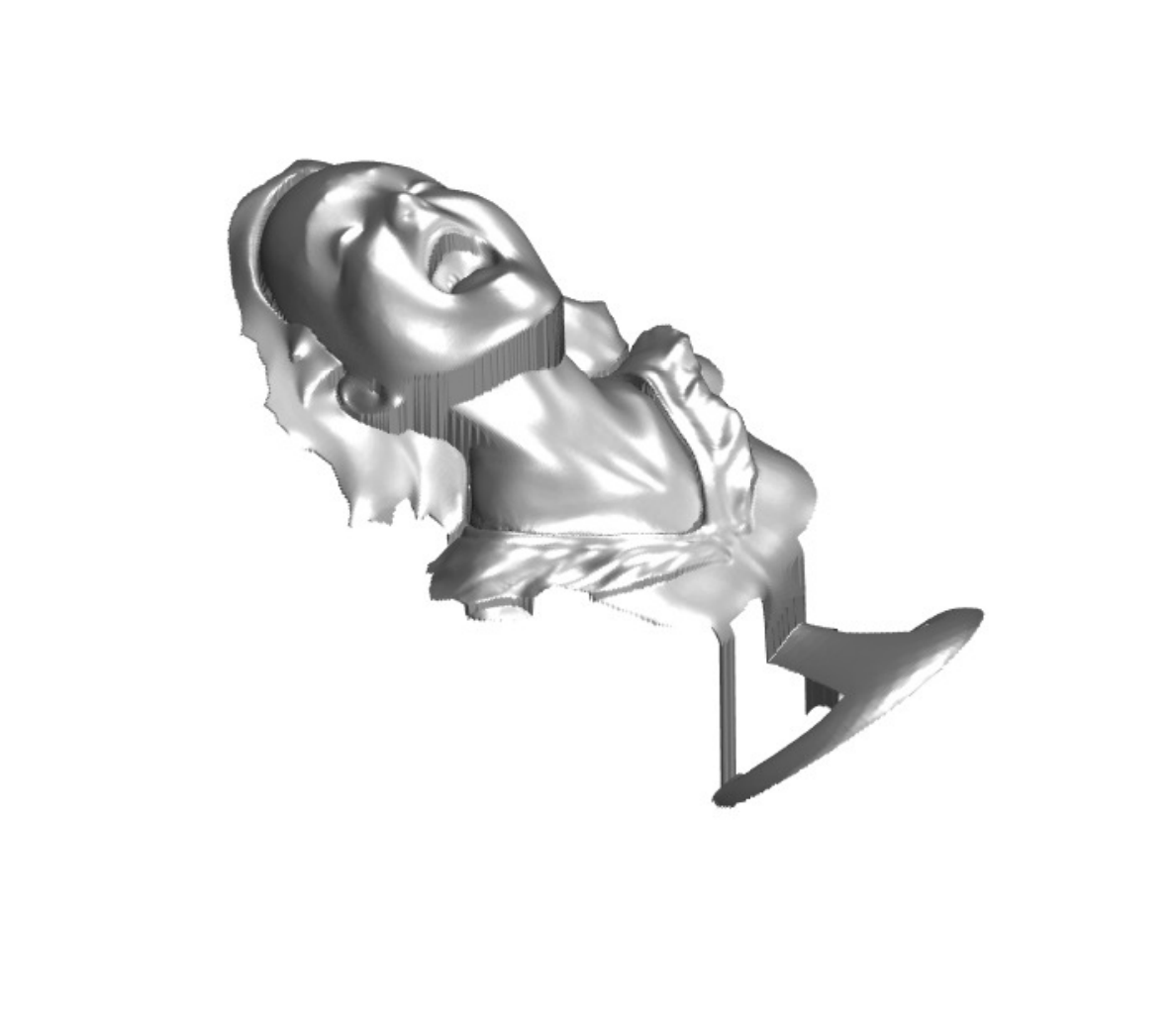} \\
  Realistic \\
   initialization \quad \end{tabular}} &
    % Ours - Dir
 \parbox[t]{2.5mm}{\rotatebox[origin=c]{90}{Ours}} &    
 \includegraphics[height=0.15\linewidth,trim = 6em 6em 6em 4em,clip]{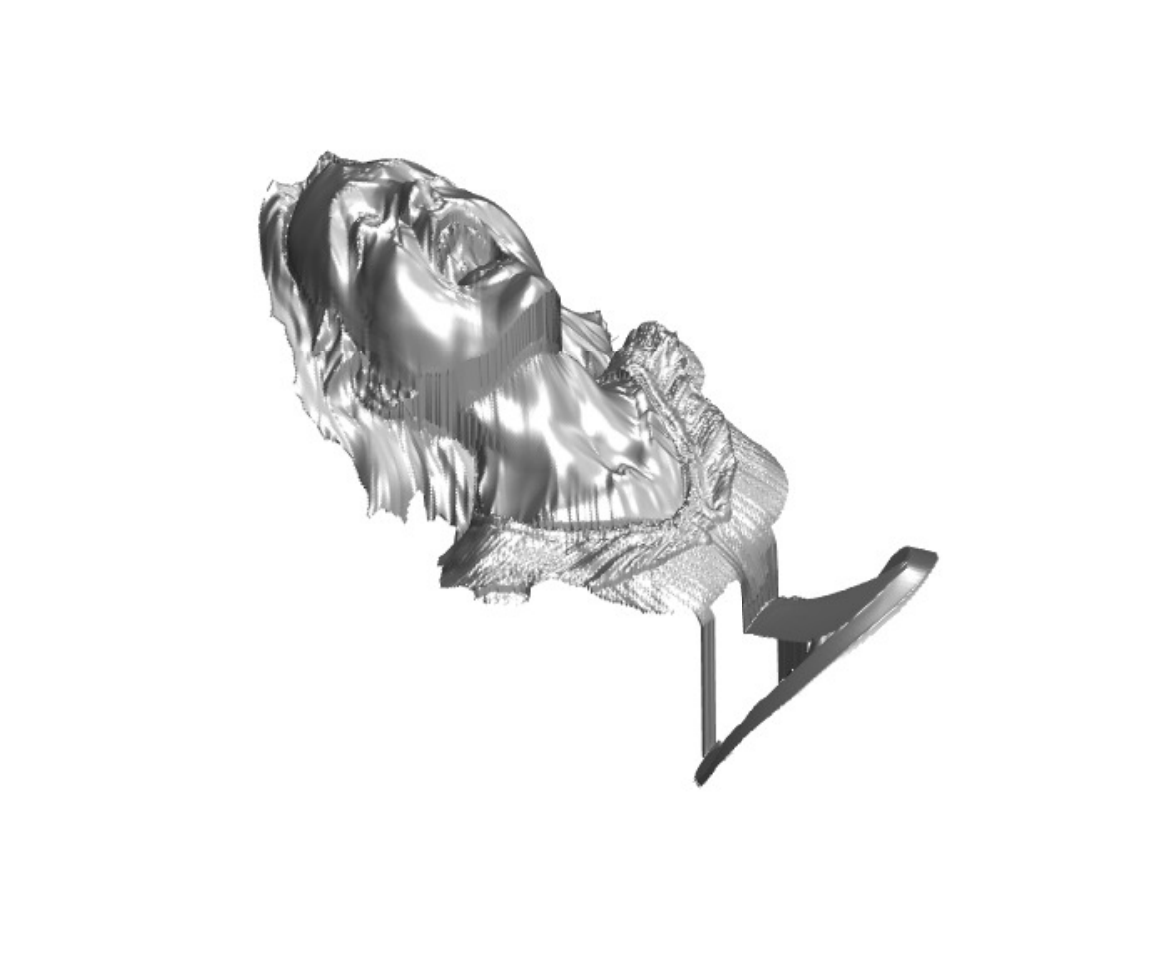}
 \includegraphics[height=0.15\linewidth]{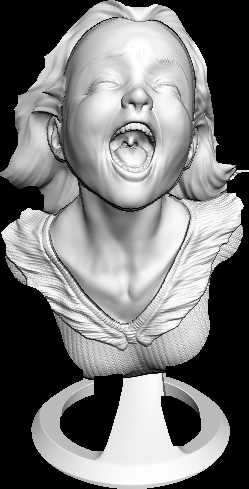} & 
  % Ours - SH2
 \includegraphics[height=0.15\linewidth,trim = 6em 6em 7em 4em,clip]{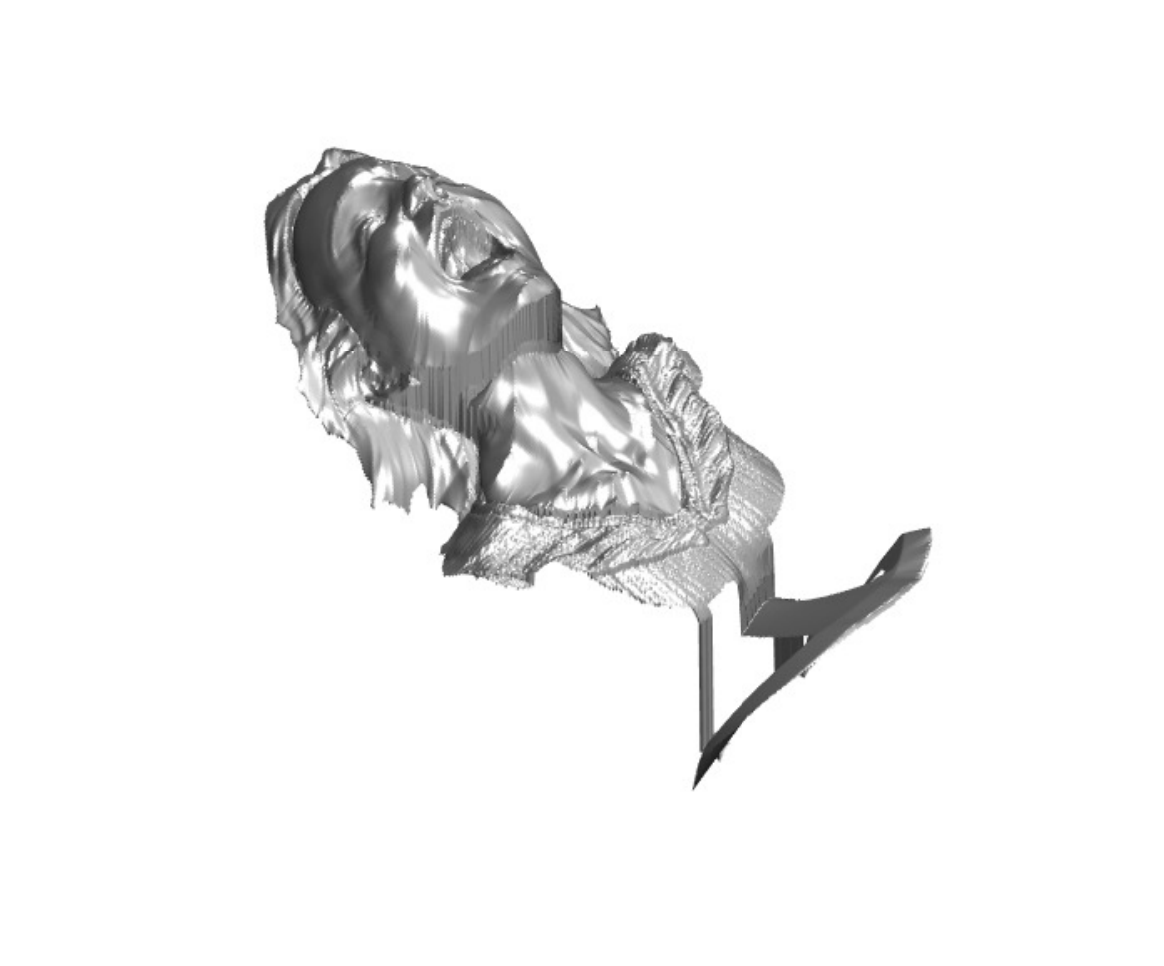} 
\includegraphics[height=0.15\linewidth]{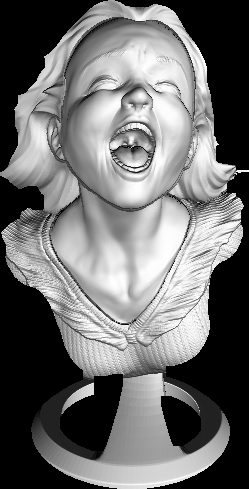} &
  % Ours - SH2 color
\includegraphics[height=0.15\linewidth,trim = 6em 6em 7em 4em,clip]{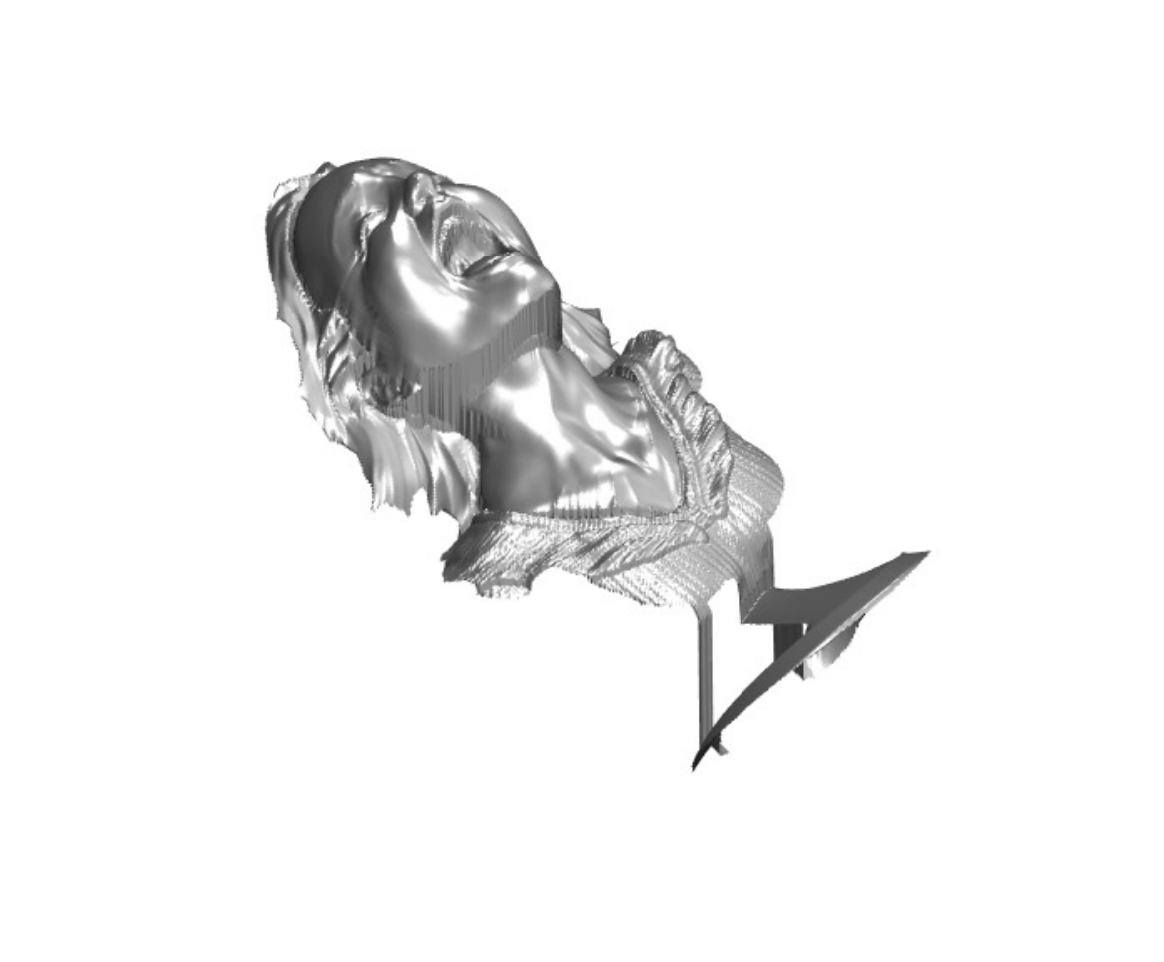} 
\includegraphics[height=0.15\linewidth]{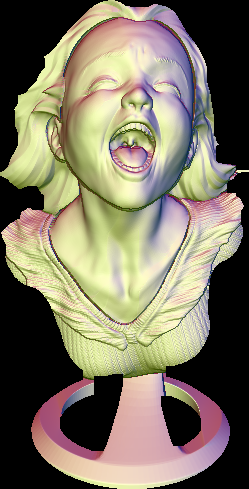} \\
 & & {\small RMSE $= \textbf{0.03}$} & {\small RMSE $= \textbf{0.03}$} & {\small RMSE $= \textbf{0.03}$} \\ 
%%%%%%%%%%%%%%%%%%
  & % Same for SIRFS
  % Ours - Dir
 \parbox[t]{2.5mm}{\rotatebox[origin=c]{90}{SIRFS}} &  
 \includegraphics[height=0.15\linewidth,trim = 6em 6em 6em 4em,clip]{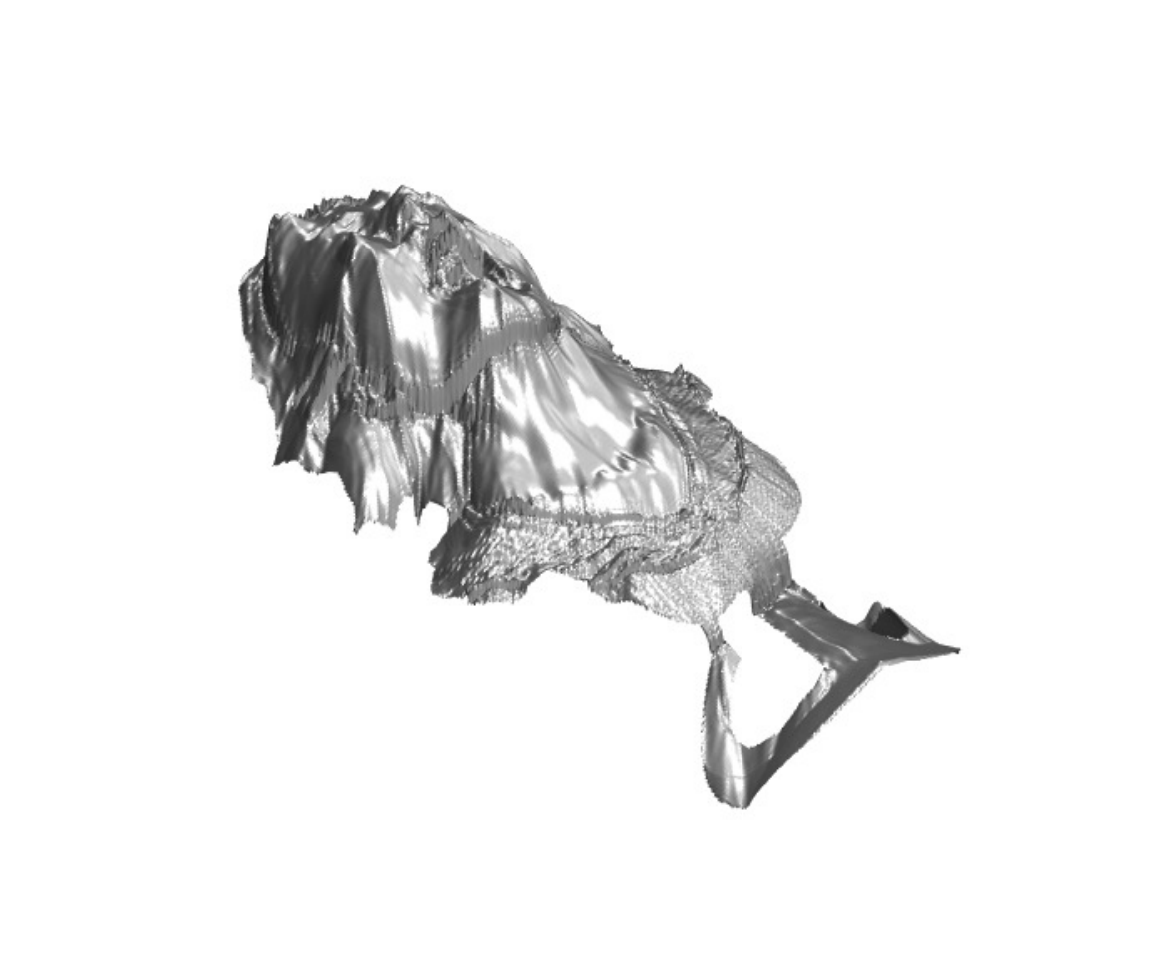}
 \includegraphics[height=0.15\linewidth]{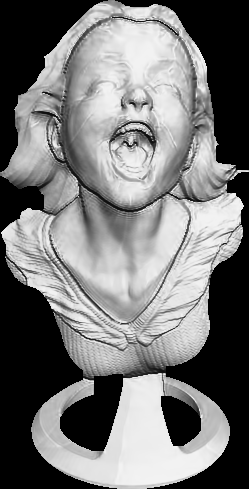} &
  % Ours - SH2
 \includegraphics[height=0.15\linewidth,trim = 6em 6em 6em 4em,clip]{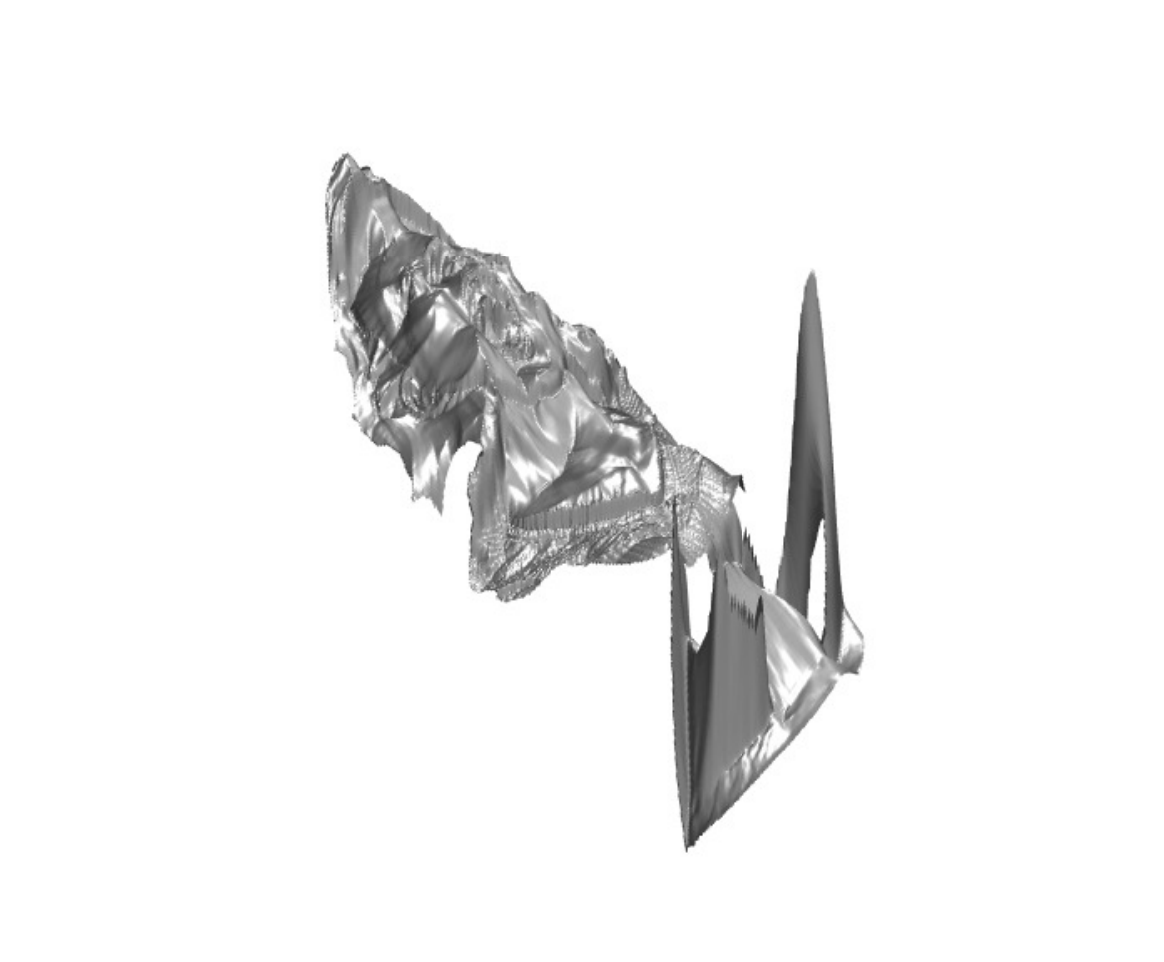} 
\includegraphics[height=0.15\linewidth]{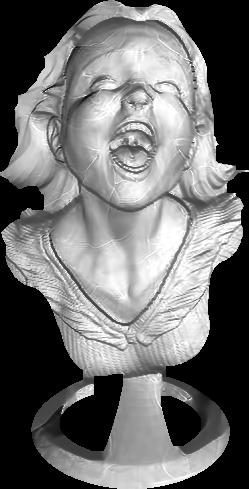} &
  % Ours - SH2 color
\includegraphics[height=0.15\linewidth,trim = 6em 6em 5em 4em,clip]{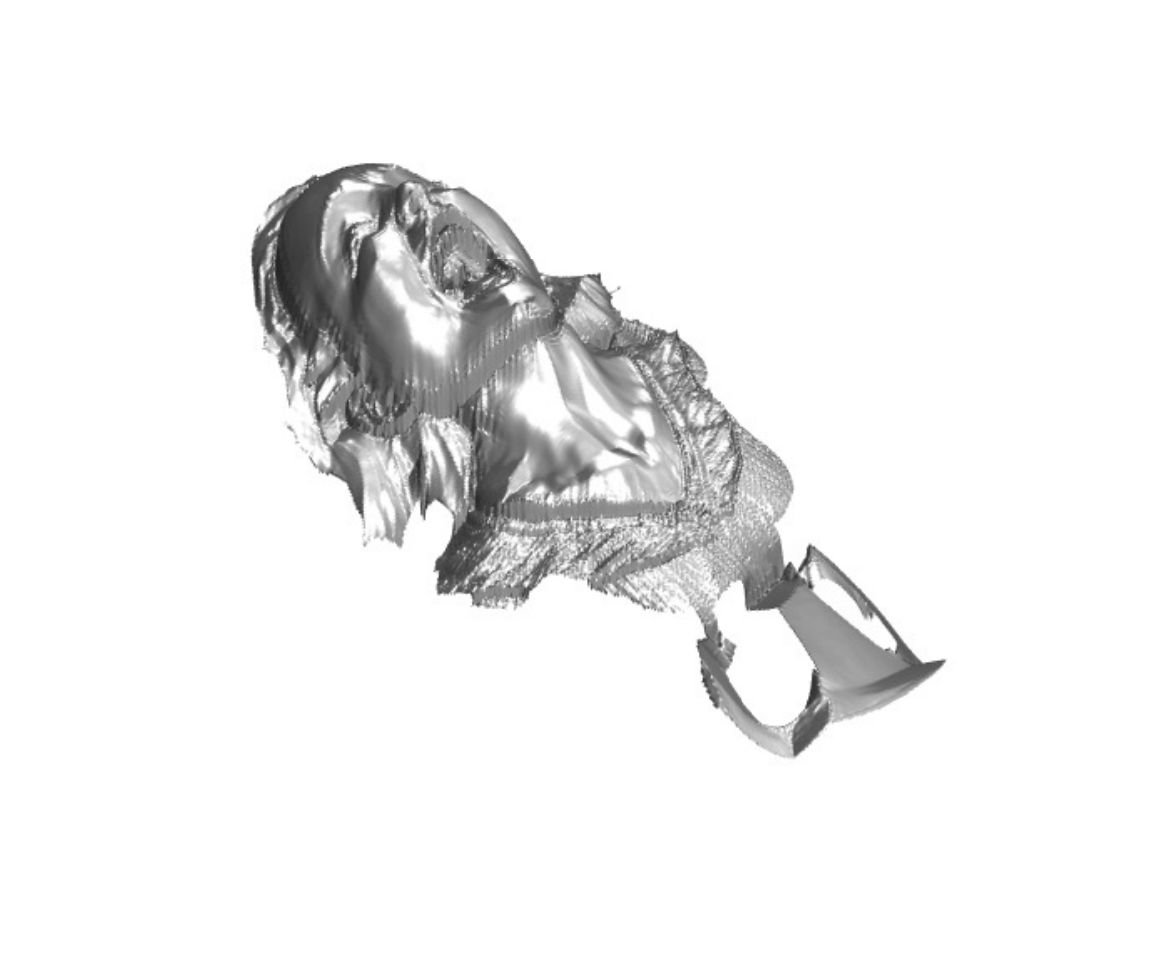} 
\includegraphics[height=0.15\linewidth]{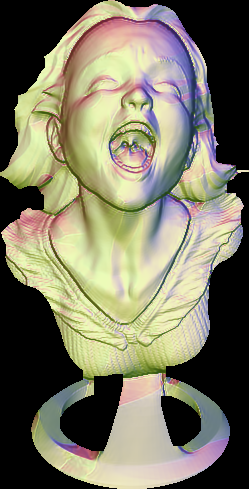} \\
 & & {\small RMSE $= 0.06$} & {\small RMSE $= 0.07$} & {\small RMSE $= 0.07$} \\ 
 \\
%%%%%%%%%%%%%%%%%%%%%%%%%%%%%%%%%%%%%%%%%%%%%%%%%%%%
  \end{tabular}
  }
  \caption{Evaluation of our SFS approach against the multi-scale one from SIRFS~\cite{Barron2015}, in three different lighting situations and using two different initial 3D-shapes (the first one is Matlab's ``peaks'' function, while the second one is a smoothed version of the ground truth). For each experiment, we show the estimated depth map and the reprojected image, and provide the root mean square error (RMSE) between the input synthetic image and the reprojection (the input images are scaled between $0$ and $1$). Our variational framework solves SFS under natural illumination more accurately than state-of-the-art.} 
  \label{fig:experiments}
\end{figure*}

To create these datasets, we use the public domain ``Joyful Yell'' 3D-shape, considering orthographic projection for fair comparison (SIRFS cannot handle perspective projection). Noise-free images are simulated under three lighting scenarios. We first consider greylevel images, with a single-order and then a second-order lighting vectors. Eventually, we consider a colored, second-order lighting vector. These lighting vectors are defined, respectively, by:
\begin{align}
  &\bl_1 =  [0.1,-0.25,-0.7,0.2,0,0,0,0,0]^\top, \label{eq:18} \\
  &\bl_2 = [0.2,0.3,-0.7,0.5,-0.2,-0.2,0.3,0.3,0.2]^\top \label{eq:19}, \\ 
&  \bl_3 = \begin{bmatrix}
    -0.2&-0.2&-1&0.4&0.1&-0.1&-0.1&-0.1&0.05 \\
    0&0.2&-1&0.3&0&0.2&0.1&0&0.1 \\
    0.2&-0.2&-1&0.2&-0.1&0&0&0.1&0
  \end{bmatrix}^{\top}.
  \label{eq:20}
\end{align}

To illustrate the underlying ambiguities, we consider two different initial 3D-shapes: one very different from the ground truth (Matlab's ``peaks'' function), and one close to it (obtained by applying a Gaussian filter to the ground truth). Interestingly, although $\mu = 0$ for the tests in Figure~\ref{fig:experiments}, our method does not drift too much from the latter: the shape is qualitatively satisfactory as soon as a good initialization is available.

In all the experiments, the images are better explained using our framework, which shows that the proposed numerical strategy solves the challenging, highly nonlinear SFS model~\eqref{eq:1} in a more accurate manner than state-of-the-art. Besides, the runtimes of both methods are comparable: a few minutes in all cases (on a standard laptop using Matlab codes), for images having around $150.000$ pixels inside $\Omega$. Unsurprisingly, initialization matters a lot, because of the inherent ambiguities of SFS.  

In Figure~\ref{fig:params}, we illustrate the influence of the hyper-parameters $\mu$ and $\nu$ which control, respectively, the shape prior and the smoothness term. We consider the same dataset as in the second experiment of Figure~\ref{fig:experiments}, but with additive, zero-mean, homoskedastic Gaussian noise on the image and on the depth forming the shape prior (we use the ``Realistic initialization'' as prior). If $\lambda = 1$ and $(\mu,\nu) = (0,0)$, then pure SFS is carried out: high-frequency details are perfectly recovered, but the surface might drift from the initial 3D-shape and interpret image noise as unwanted geometric artifacts. If $\mu \to +\infty$, the initial estimate (which exhibits reasonable low-frequency components, but no geometric detail) is not modified. If $\nu \to +\infty$, then only the minimal surface term matters, hence the result is over-smoothed. In this experiment, we also evaluate the accuracy of the 3D-reconstructions through the mean angular error (MAE) on the normals: it is minimal when the parameters are tuned appropriately, not when the image error (RMSE) is minimal since minimizing the latter comes down to estimating geometric details explaining the image noise.  

The appropriate tuning of $\mu$ and $\nu$ depends on how trustworthy the image and the shape prior are. Typically, in RGB-D sensing, the depth may be noisier than in this synthetic experiment: in this case a low value of $\mu$ should be used. On the other hand, natural illumination is generally colored, so the three image channels provide redundant information: regularization is less important and a low value of the smoothness parameter $\nu$ can be used. We found that $(\lambda,\mu,\nu) = (1,1,5.10^{-5})$ provides qualitatively nice results in all our real-world experiments. 

\begin{figure}[!ht]
\centering
\begin{tabular}{ccccc}
 & $\nu = 10^{-6}$ & $\nu = 10^{-3}$ & $\nu = 1$ & \\
\multirow{2}{*}{ \includegraphics[valign=m,width = 0.13\linewidth]{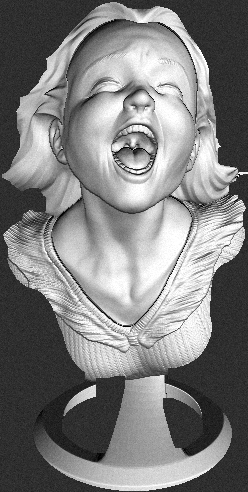} }  
  \quad & \quad  \includegraphics[valign=m,width = 0.13\linewidth]{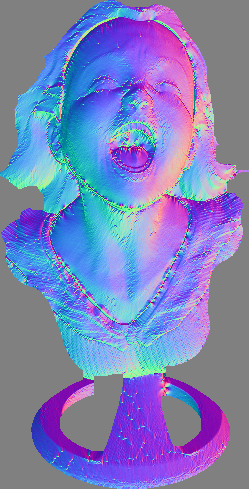} \quad&\quad
 \includegraphics[valign=m,width = 0.13\linewidth]{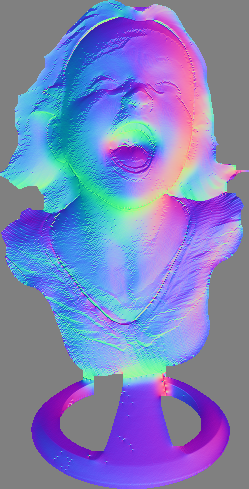} \quad&\quad
 \includegraphics[valign=m,width = 0.13\linewidth]{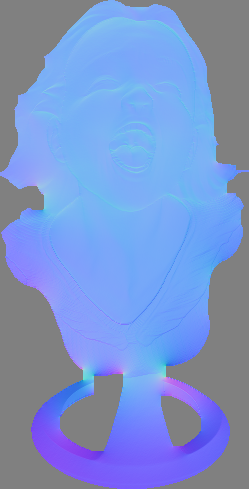} \quad &\quad $\mu = 10^{1}$ \\
 \quad&\quad {\small RMSE $ = \textbf{0.06}$} \quad&\quad {\small RMSE $ = 0.12$} \quad&\quad {\small RMSE $ = 0.29$} \quad&\quad \\
 \quad&\quad {\small MAE $ = 14.77$} \quad&\quad {\small MAE $ = 18.56$} \quad&\quad {\small MAE $ = 41.91$} \quad&\quad \\  
 \quad&\quad \includegraphics[valign=m,width = 0.13\linewidth]{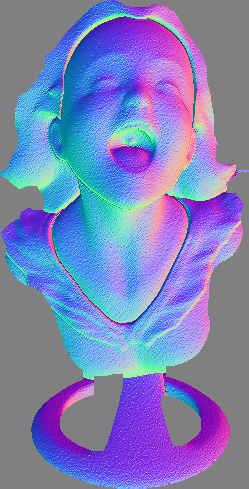} \quad&\quad
 \includegraphics[valign=m,width = 0.13\linewidth]{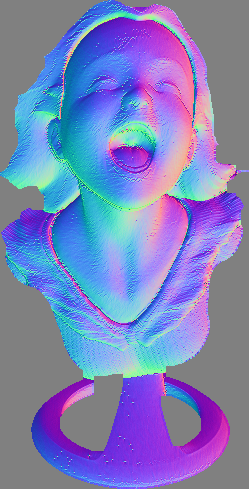} \quad&\quad
 \includegraphics[valign=m,width = 0.13\linewidth]{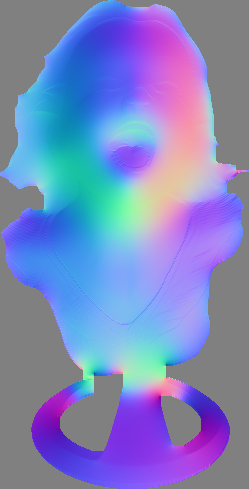} \quad&\quad $\mu = 10^{3}$ \\  
\multirow{2}{*}{\includegraphics[valign=m,width = 0.13\linewidth]{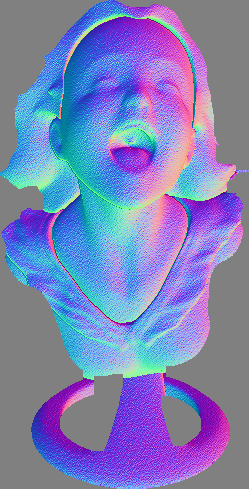} }  
 \quad&\quad {\small RMSE $ = 0.08$} \quad&\quad {\small RMSE $ = 0.13$} \quad&\quad {\small RMSE $ = 0.20$}  \quad&\quad \\
 \quad&\quad {\small MAE $ = 18.34$} \quad&\quad {\small MAE $ = \textbf{13.46}$} \quad&\quad {\small MAE $ = 24.43$} \quad&\quad \\   
\quad&\quad  \includegraphics[valign=m,width = 0.13\linewidth]{Fig/SelectoImitationCoca/normal_mu_10_nu_2} \quad&\quad
 \includegraphics[valign=m,width = 0.13\linewidth]{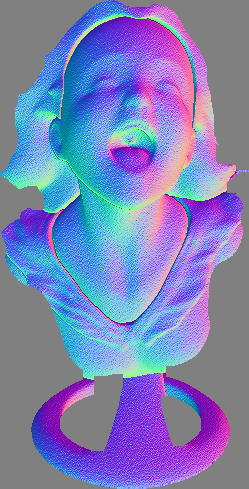} \quad&\quad
 \includegraphics[valign=m,width = 0.13\linewidth]{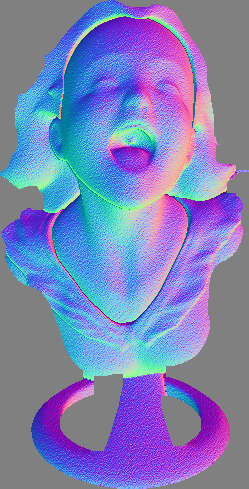} \quad&\quad $\mu = 10^{5}$ \\
 \quad&\quad {\small RMSE $ = 0.17$} \quad&\quad {\small RMSE $ = 0.17$} \quad&\quad {\small RMSE $ = 0.16$} \quad&\quad \\
 \quad&\quad {\small MAE $ = 21.91$} \quad&\quad {\small MAE $ = 21.91$} \quad&\quad {\small MAE $ = 21.02$} \quad&\quad \\         
\end{tabular}
  \caption{Left: input noisy image ($\sigma_I = 2\%$ of the maximum greylevel) and noisy prior shape ($\sigma_z = 0.2 \%$ of the maximum depth), represented by a normal map to emphasize the details. Right: estimated shape with $\lambda = 1$ and various values of $\mu$ and $\nu$. The RMSE between the image and the reprojection is minimal when $\mu$ and $\nu$ are minimal, but the mean angular error (MAE, in degrees) between the estimated shape and the ground truth one is not.}
  \label{fig:params}
\end{figure}

\subsection{Qualitative Evaluation on Real-world Datasets}
 
The importance of initialization is further confirmed in the top rows of Figures~\ref{fig:teaser} and~\ref{fig:cushion_SFS}. In these experiments, our SFS method  ($\mu = \nu = 0$) is evaluated, under perspective projection,  on real-world datasets obtained using an RGB-D sensor~\cite{Han2013}, considering a fronto-parallel surface as initialization. Although fine details are revealed, the results present an obvious low-frequency bias, and artifacts due to the image noise occur. This illustrates both the inherent ambiguities of SFS, and the need for depth regularization.   

\begin{figure}[!ht]
\centering
\begin{tabular}{ccc}
    \includegraphics[height = 0.25\linewidth]{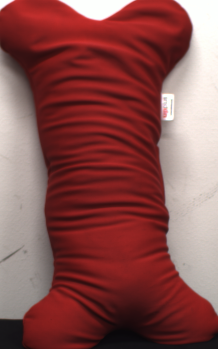} \hspace*{-3.3em}\includegraphics[height = 0.08\linewidth]{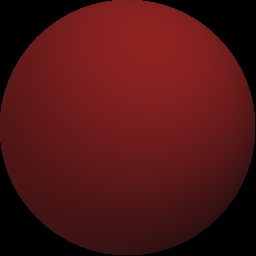}  &\, 
    \includegraphics[height = 0.25\linewidth]{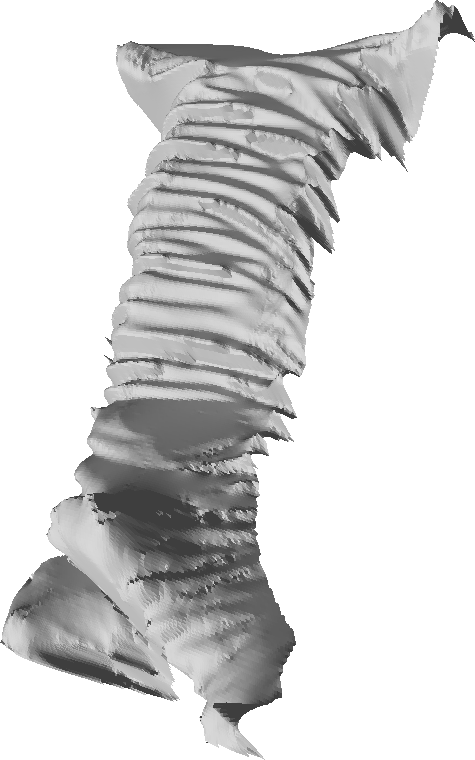} &\, 
    \includegraphics[height = 0.25\linewidth]{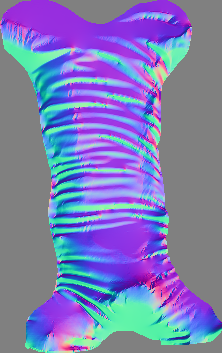} \\
 {\small Input real image with illumination~\cite{Han2013} } \qquad  & \multicolumn{2}{c}{\small SFS result (no regularization)}\\
 & \multicolumn{2}{c}{ ($(\lambda,\mu,\nu)=(1,0,0)$)}\\[.5em]
\end{tabular}
 \begin{tabular}{cccc}
    \includegraphics[height = 0.25\linewidth]{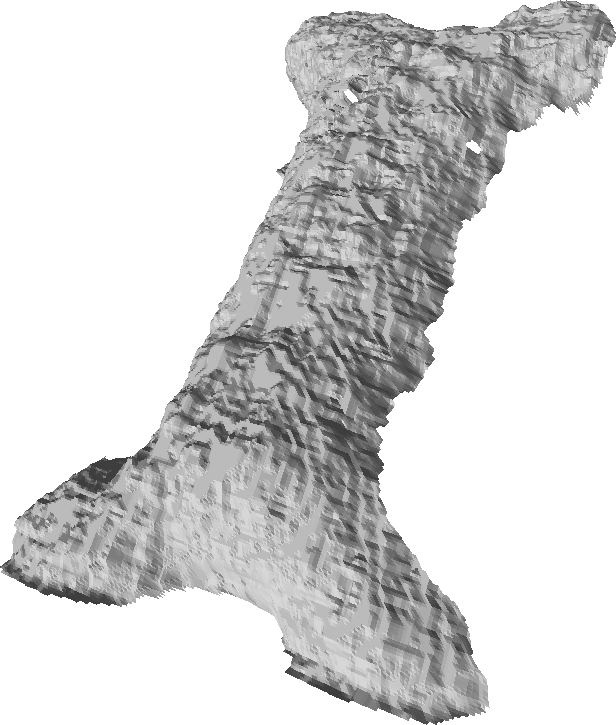}  &\, 
    \includegraphics[height = 0.25\linewidth]{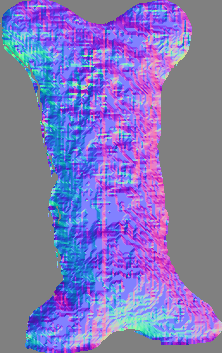}  &\quad     
    \includegraphics[height = 0.25\linewidth]{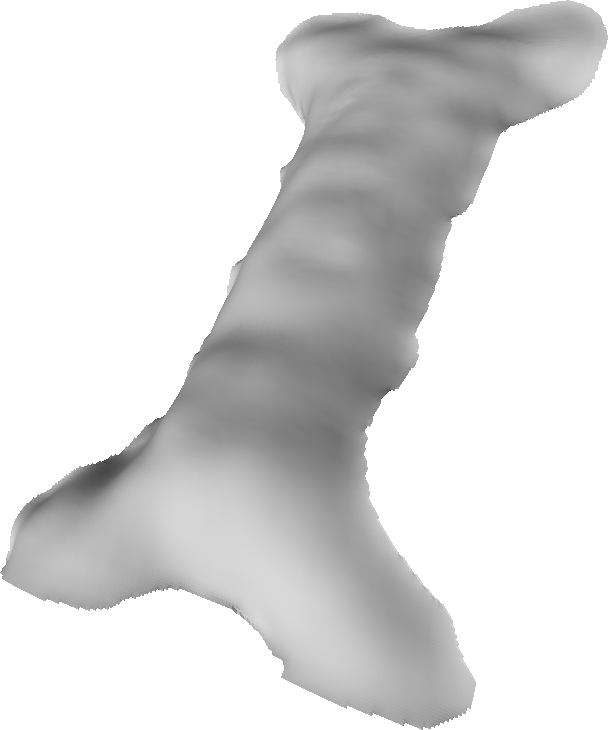} &\,
    \includegraphics[height = 0.25\linewidth]{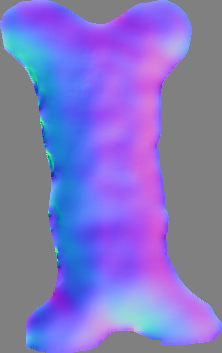} \\
\multicolumn{2}{c}{\small Noisy input shape and normals~\cite{Han2013} }    \qquad   & \multicolumn{2}{c}{\small Minimal surface denoising (no SFS)}  \\
    & & \multicolumn{2}{c}{\small ($(\lambda,\mu,\nu) = (0,1,5.10^{-5})$) }\\[.5em] 
  \end{tabular}
 \begin{tabular}{cc}
    \includegraphics[height = 0.325\linewidth]{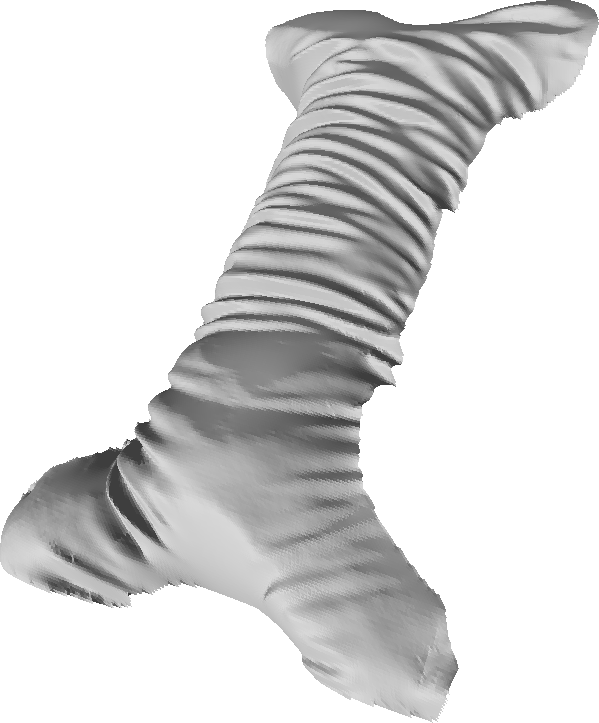}  &\, 
    \includegraphics[height = 0.325\linewidth]{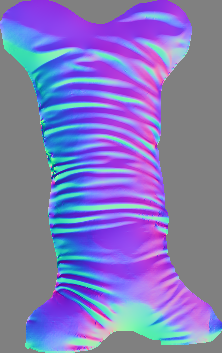}  \\     
    \multicolumn{2}{c}{\small SFS-based denoising and completion ($(\lambda,\mu,\nu) = (1,1,5.10^{-5})$) } 
    %  & {\small Relighting }
  \end{tabular}  
  \caption{Results on three computer vision problems: SFS, ``blind'' (not shading-based) depth refinement, and shading-based depth refinement. The shape estimated by SFS is distorted (due to the ambiguities of SFS), and artifacts occur (due to noise), but it contains the fine-scale details. Although the depth map provided by the RGB-D sensor is denoised without considering shading, thin structures are missed. With the proposed method, noise is removed and fine details are revealed.
  % allowing the refined shape to be be used in relighting applications (here, we change the viewpoint, the lighting and the reflectance).
  }
  \label{fig:cushion_SFS}
\end{figure}

In order to illustrate the practical disambiguation of SFS using a shape prior, we next consider as initialization $z^{(0)}$ and prior $z^0$ the depth provided by the RGB-D sensor. It is both noisy and incomplete, but with our framework it can be denoised, refined and completed in a shading-aware manner, by tuning the parameters $\mu$ (prior) and $\nu$ (smoothness). Second and third rows of Figures~\ref{fig:teaser} and~\ref{fig:cushion_SFS} illustrate the interest of SFS for depth refinement, in comparison with ``blind'' methods based solely on depth regularization~\cite{Graber2015}. 

Eventually, Figure~\ref{fig:figure} demonstrates an application to stereovision, using a real-world dataset from~\cite{Zollhoefer2015}. This time, the initial depth map is obtained by a multi-view stereo (MVS) algorithm~\cite{Jancosek2011}. We estimated lighting from this initial depth map, assuming uniform albedo. Then, we let our algorithm recover the thin geometric structures, which are missed by MVS. The initial depth map contains a lot of missing data and discontinuities, which is challenging for our algorithm: ambiguities arise inside the large holes, and our model favors smooth surfaces. Indeed, the concavities are not well recovered, and the discontinuities are partly smoothed. Still, nice details are recovered, and the overall surface seems reasonable.

\begin{figure}[!ht]
\centering
  \setlength{\tabcolsep}{0.1em}
  \begin{minipage}{0.27\linewidth}
  \centering
  \begin{tabular}{c}
   \includegraphics[width=0.8\linewidth]{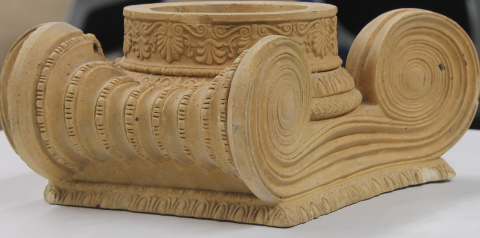}\hspace*{-1.01em}\includegraphics[height=0.1\linewidth]{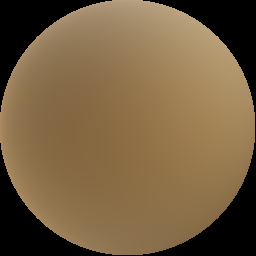} \\ 
   {\small $I_1$} \\
   \includegraphics[width=0.7\linewidth]{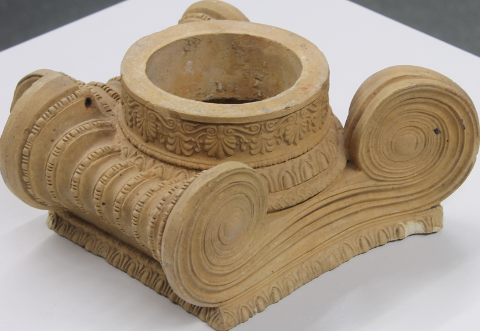}\hspace*{-1.01em}\includegraphics[height=0.1\linewidth]{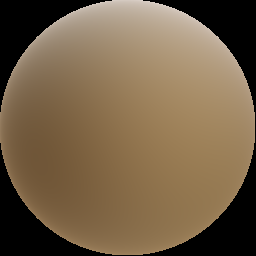} \\
      {\small $I_2$} 
  \end{tabular}
  \end{minipage}
  \begin{minipage}{0.72\linewidth}
  \centering
  \begin{tabular}{cc}
   \includegraphics[width=0.495\linewidth]{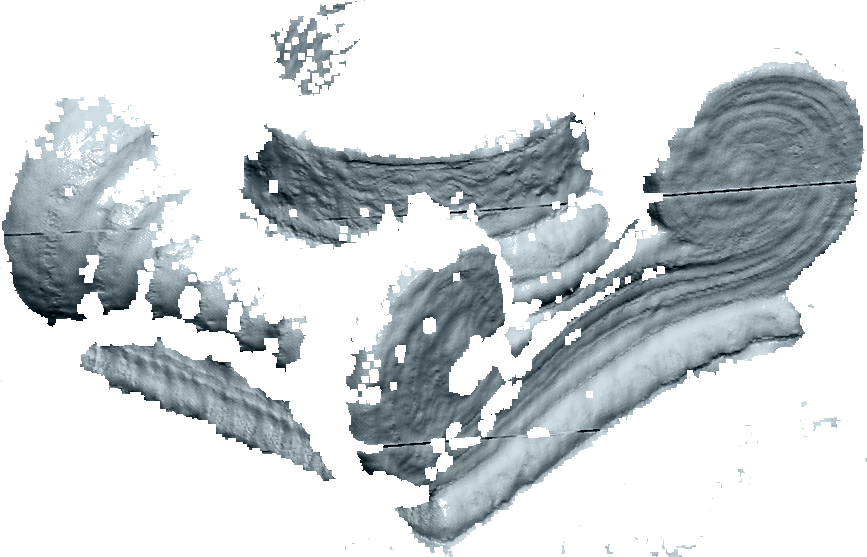} & 
   \includegraphics[width=0.495\linewidth]{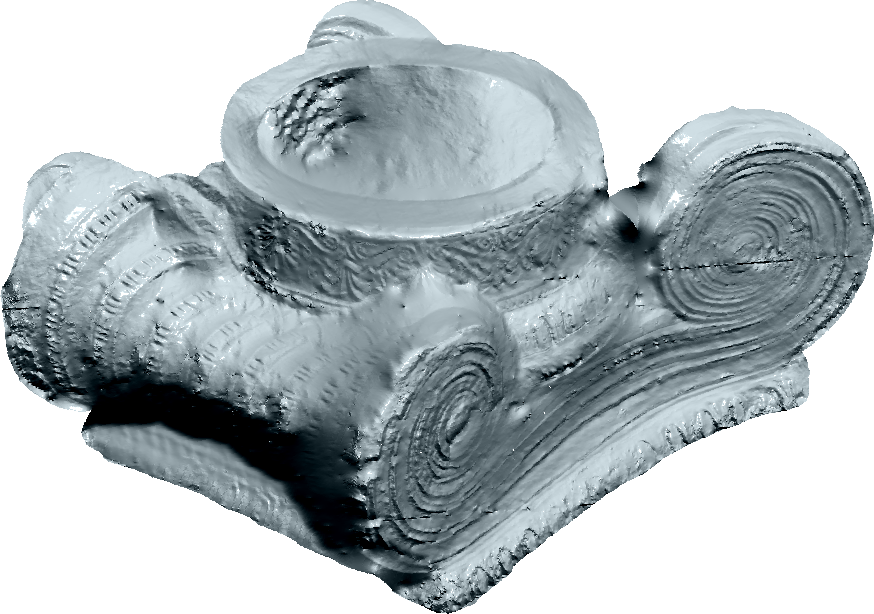} \\ 
% {\small Prior depth map $z^{0,2}$~\cite{Jancosek2011}} & {\small Refined depth map $z^2$}
 {\small Prior depth map $z_{2}^0$~\cite{Jancosek2011}} & {\small Refined depth map $z_2$}
  \end{tabular}
  \end{minipage}  
  \vskip-1em
%\caption{Left: two (out of $N=30$) images $I^1$ and $I^2$ of the ``Figure'' object~\cite{Zollhoefer2015}. Middle: depth map $z^{0,2}$ (same viewpoint as image $I^2$) obtained by the CMPMVS method~\cite{Jancosek2011} (before depth maps are processed for meshing). Right: refined and completed depth map $z^2$.\vspace*{-2.2em}}
\caption{Left: two (out of $30$) images $I_1$ and $I_2$ of the ``Figure'' object~\cite{Zollhoefer2015}. Middle: depth map $z_{2}^0$ obtained by the CMPMVS method~\cite{Jancosek2011} (before meshing). Right: refined and completed depth map $z_2$.}
\label{fig:figure}
\end{figure}

\section{Conclusion and Perspectives}
\label{sec:6}

We have introduced a generic variational framework for SFS under natural illumination, which can be applied in a broad range of scenarios. It relies on a tailored PDE-based SFS formulation which handles a variety of models for the camera and the lighting. To solve the resulting system of PDEs, we introduce an ADMM algorithm which separates the difficulty due to nonlinearity from that due to the dependency upon the gradient. Shape prior and nonlinear smoothing terms are easily included in this variational framework, allowing disambiguation of SFS as well as practical applications to depth map refinement and completion for RGB-D sensors or stereovision systems. 

As future work, we plan to investigate the convergence of the proposed ADMM algorithm for our non-convex problem, and to include reflectance and lighting estimation. With these extensions, we have good hope that the proposed variational framework may be useful in other computer vision applications, such as shading-aware dense multi-view stereo.

\bibliographystyle{splncs} 
\bibliography{biblio}

\end{document}